%% file: main.tex
\documentclass[nohyperref]{article}

\usepackage{microtype}
\usepackage{graphicx}
\usepackage{subfigure}
\usepackage{booktabs} 
\usepackage{hyperref}

\usepackage[accepted]{icml2022}

\usepackage{amsmath}
\usepackage{amssymb}
\usepackage{mathtools}
\usepackage{amsthm}
\usepackage{bbm}
\usepackage[capitalize,noabbrev]{cleveref}

\theoremstyle{plain}
\usepackage{xcolor}
\usepackage{thm-restate}

\DeclareMathOperator*{\argmax}{arg\,max}

\declaretheorem[name=Definition,numberwithin=section]{definition}

\newcommand{\updatematrix}{A}
\newcommand{\xtest}{x_{\mathrm{test}}}

\newcommand{\Xtest}{X_{\mathrm{test}}}
\newcommand{\Xtrain}{X_{\mathrm{train}}}
\newcommand{\train}{{\mathrm{train}}}
\newcommand{\test}{{\mathrm{test}}}
\newcommand{\ttheta}{\tilde{\theta}}
\newcommand{\TD}{\mathrm{TD}}
\newcommand{\states}{\mathcal{X}}
\newcommand{\mdp}{\mathcal{M}}
\usepackage[textsize=tiny]{todonotes}

\usepackage{mathrsfs}
\newcommand{\markcomment}[1]{\text{}}
\newcommand{\clarecomment}[1]{\text{}}
\newcommand{\willcomment}[1]{\text{}}

\icmltitlerunning{Learning Dynamics and Generalization in Reinforcement Learning}

\begin{document}

\twocolumn[
\icmltitle{Learning Dynamics and Generalization in Reinforcement Learning}
\icmlsetsymbol{equal}{*}

\begin{icmlauthorlist}
\icmlauthor{Clare Lyle}{oxf}
\icmlauthor{Mark Rowland }{dpm}
\icmlauthor{Will Dabney}{dpm}
\icmlauthor{Marta Kwiatkowksa}{oxf}
\icmlauthor{Yarin Gal}{oxf}

\end{icmlauthorlist}

\icmlaffiliation{oxf}{Department of Computer Science, University of Oxford}
\icmlaffiliation{dpm}{DeepMind}

\icmlcorrespondingauthor{Clare Lyle}{clare.lyle@cs.ox.ac.uk}

\icmlkeywords{Machine Learning, ICML}

\vskip 0.3in
]



\printAffiliationsAndNotice{}  

\begin{abstract}
Solving a reinforcement learning (RL) problem poses two competing challenges: fitting a potentially discontinuous value function, and generalizing well to new observations. In this paper, we analyze the learning dynamics of temporal difference algorithms to gain novel insight into the tension between these two objectives.
We show theoretically that temporal difference learning encourages agents to fit non-smooth components of the value function early in training, and at the same time induces the second-order effect of discouraging generalization.
We corroborate these findings in deep RL agents trained on a range of environments, finding that neural networks trained using temporal difference algorithms on dense reward tasks exhibit weaker generalization between states than randomly initialized networks and networks trained with policy gradient methods.
Finally, we investigate how post-training policy distillation may avoid this pitfall, and show that this approach improves generalization to novel environments in the ProcGen suite and improves robustness to input perturbations.
\end{abstract}

\section{Introduction}
The use of function approximation in reinforcement learning (RL) faces two principal difficulties: existing algorithms are vulnerable to divergence and instability \citep{baird1993advantage}, and value estimates that do converge tend to generalize poorly to new observations \citep{zhang2018study}. Crucial to both of these difficulties is the concept of \textit{interference}, the degree to which an update to an agent's predictions at one state influences its predictions at other states. Function approximation schemes with weaker interference, such as those induced by tabular value functions or tile coding schemes, have been shown empirically to produce more stable behaviour and faster convergence in value-based algorithms on a number of classic control domains \citep{ghassian2020improving}. However, such schemes by construction require treating the value functions for different states independently, limiting the potential for a model’s predictions to generalize to new observations. We will be interested in studying how the learning dynamics induced by typical update rules such as Q-learning and policy gradients influence interference in deep RL agents. 

An array of prior empirical works demonstrate that RL algorithms which succeed at the tasks they are trained on frequently overfit on their training environment's observations and dynamics \citep{lewandowskigeneralization,farebrother2018generalization,cobbe2021phasic, zhang2018study}.
While many prior works have sought training methodologies to improve generalization \citep{igl2019generalization, raileanu2021automatic}, the source of the relative tendency of deep RL methods to overfit to their training distribution remains under-explored. Understanding the role of interference in deep reinforcement learning may thus provide insights that lead to the development of agents which can quickly learn robust policies in any application domain. One avenue that we will draw on in particular is the study of agents' \textit{learning dynamics}, which can reveal insights into not just the convergence of algorithms, but also into the trajectory taken by the agent's value function.

In this work, we study how interference evolves in deep RL agents. Our primary contributions will be twofold: first, to provide a rigorous theoretical and empirical analysis of the relationship between generalization, interference, and the dynamics of temporal difference learning; second, to study the effect of distillation, which avoids the pitfalls of temporal difference learning, on generalization to novel environments. 
Towards this first contribution, we extend the analysis of \citet{lyle2021effect} to show that the dynamics of temporal difference learning accelerate convergence along non-smooth components of the value function first, resulting in implicit regularization towards learned representations that generalize weakly between states. Our findings present an explanation for widely-observed vulnerability of value-based deep RL agents to overfit to their training observations \citep{raileanu2021decoupling,zhang2018study}.  

We then evaluate whether these findings hold empirically across a range of popular deep RL benchmarks. We measure interference by constructing a summary statistic which evaluates the extent to which optimization steps computed for one state influence predictions on other states, which we call the \textit{update rank}. We find that value-based agents trained with temporal difference (TD) methods learn representations with weak interference between states, performing updates similar to those of a lookup table, whereas networks trained with policy-gradient losses learn representations for which an update on one state has a large effect on the policy at other states.
Finally, we show that post-training policy distillation is a cheap and simple approach to improve the generalization and robustness of learned policies.
\section{Background}
We focus on the reinforcement learning problem, which we formalize as a Markov decision process (MDP) \citep{puterman}. An MDP consists of a tuple $\langle \states, A, R, P, \gamma, \states_0 \rangle$, where $\states$ denotes the set of states, $A$ the set of actions, $R:\states \times A \rightarrow \mathbb{R}$ a reward function, $P: \states \times A \rightarrow \mathscr{P}(\states)$ a possibly stochastic transition probability function, $\gamma \in [0, 1)$ the discount factor, and $\states_0$ the initial state. In reinforcement learning, we seek an optimal policy $\pi^* : \states \rightarrow \mathscr{P}(A)$ which maximizes the expected sum of discounted returns from the initial state. 

\textbf{Value-based reinforcement learning} seeks to model the action-value function 
\begin{equation}
    Q^{\pi^*}(x,a) = \mathbb{E}[\sum_{t \geq 0} \gamma^t R_t(x_t, a_t)|x_0=x, a_0=a ]
\end{equation} 
as a tool to learn an optimal policy. 
Given a policy $\pi$, we leverage a recursive expression of the value function as the fixed point of the policy evaluation Bellman operator, defined as follows
\begin{equation}
   (T^\pi Q)(x,a) = \mathbb{E}_{P(x' | x,a), \pi(a'| x')}[R(x,a) + \gamma  Q(x', a')] .
\end{equation}
Temporal difference learning \citep{sutton1988learning} performs updates based on sampled transitions ($x_t, a_t, r_t, x'_t$), leveraging a stochastic estimate of the Bellman targets. 

To find an optimal policy, we turn to the control setting. We let the policy $\pi_t$, used to compute the target, be greedy with respect to the current action-value function $Q_t$. This results in updates based on the Bellman \textit{optimality} operator
\begin{equation}
(T^* Q_t)(x,a)  = \mathbb{E} [ R(x,a) + \gamma  \max_{a'} [Q^{\pi_t}(x', a')]] \; .
\end{equation}
In control problems, Q-learning \citep{watkins1989learning,watkins1992q} is a widely-used stochastic approximation of the Bellman optimality operator. When a function approximator $Q_\theta$, with parameters $\theta$, is used to approximate $Q^\pi$, as in deep reinforcement learning, we perform semi-gradient updates $f(\theta)$ of the following form, where $a^* = \argmax_{a}(Q_\theta(x_{t+1}, a))$.
\begin{equation}
    f(\theta) = (\nabla_\theta Q_\theta) [r_t + \gamma Q_\theta(x_{t+1}, a^*) - Q_\theta(x_t, a_t)]
\end{equation}

\textbf{Policy gradient methods} \citep{sutton2000policy} operate directly on a parameterized policy $\pi_\theta$. We let $d^\pi$ denote the stationary distribution induced by a policy $\pi$ over states in the MDP. Policy gradient methods aim to optimize the parameters $\theta$ of the policy so as to maximize the objective $J(\pi_\theta) = \mathbb{E}_{x\sim d^{\pi_\theta}} \sum_{a} \pi_\theta(a|x) Q^{\pi_\theta}(x,a) $. This objective can be maximized by following the gradient
\begin{equation}
    \nabla_\theta J(\pi_\theta) = \mathbb{E}_{x_t, a_t}[ \nabla_\theta \log \pi_\theta(a_t|x_t) Q^{\pi_\theta}(x_t, a_t)] \;.
\end{equation}
Variations on this learning rule include \textit{actor-critic} methods \citep{konda2000actor}, which use a baseline given by a value-based learner to reduce update variance, and trust-region based methods, such as TRPO \citep{schulman2015trust} and PPO \citep{schulman2017proximal}. 
  
\textbf{Generalization} arises in reinforcement learning in the context of solving new reward functions \citep{dayan1993improving}, and in the case of large observation spaces or procedurally-generated environments, where some degree of generalization to new observations is necessary in order to obtain good performance at deployment \citep{kirk2021survey}. We will be concerned in this paper with the generalization gap incurred by a policy learned on a set of environments $\mathcal{E}_{\Xtrain}$ on the \textit{test} environment.
\begin{equation}
   \mathbb{E}_{\mathcal{E}_{\Xtrain}, \pi} [ \sum_{t=0}^\infty \gamma^t R(x_t, a_t)] - \mathbb{E}_{\mathcal{E}_{\Xtest}, \pi}[ \sum_{t=0}^\infty \gamma^t R(x_t, a_t)]
\end{equation}
In general, $\mathcal{E}_{\Xtest}$ will be assumed to share some structure with $\mathcal{E}_{\Xtrain}$. In large observation spaces, $\mathcal{E}_{\Xtest}$ may be equal to $\mathcal{E}_{\Xtrain}$ with a different initial state distribution, while for multi-environment problems $\mathcal{E}_{\Xtrain}$ may be homomorphic to $\mathcal{E}_{\Xtest}$ \citep{zhang2020invariant}.
A necessary condition for generalization to new observations is that a parameter update computed on a training state changes the agent's policy on the unseen test states, a phenomenon we will refer to as \textit{interference}. While interference is defined in many ways in the literature, we use it here to refer to the effect of an optimization step performed on the loss of one transition to change the agent's predicted value of other states in the environment, following a usage similar to \citet{bengio2020interference}.
\section{Related work}
\textbf{Generalization in RL.}
A number of prior works have presented methods to quantify and improve generalization in reinforcement learning \citep{igl2019generalization, raileanu2021automatic, hansen2021generalization, laskin2020reinforcement, yarats2020image,wang2020improving, cobbe2020leveraging, kenton2019generalizing}. The study of generalization in deep RL has focused principally on overfitting to limited observations \citep{song2019observational}, and generalization to novel environments \citep{farebrother2018generalization, cobbe2019quantifying}. Work on generalization in deep learning more broadly has shown that neural networks are biased towards `simple' functions, for varying notions of simplicity \citep{perez2018deep, hochreiter1995simplifying, fort2020deep, izmailov2018averaging, farnia2020fourier}. The study of  this bias in reinforcement learning tasks \citep{yang2022overcoming}, has demonstrated that the bias of neural networks to smooth functions can harm value approximation accuracy in deep RL, and proposes tuning the scale of learnable Fourier features as one means of ensuring high-frequency components of the value function are captured, an approach also followed by \citet{brellmann2022fourier}.
\citet{raileanu2021decoupling} further highlight that the process of learning a value function can induce overfitting, improving generalization to novel environments by decoupling value approximation and policy networks in actor critic architectures.

\textbf{Interference and stability.}
Off-policy temporal difference learning is not guaranteed to converge in the presence of function approximation \citep{baird1993advantage, tsitsiklis1997analysis}, the setting under which deep RL algorithms are most commonly used. The key driver of instability is interference, which has been studied in settings ranging from bandits \citep{schaul2019ray} to deep RL \citep{fedus2020catastrophic, achiam2019towards}. A number of approaches which specifically reduce the effect of a gradient update for state $s$ on the target $V(s')$ have been shown to improve the stability and robustness of these methods \citep{ghassian2020improving, lo2019overcoming}. Many prior works have also endeavoured to define and analyze interference in deep RL \citep{liu2020measuring, liu2020towards, bengio2020interference}, and to study its role in the stability of offline algorithms \citep{kumar2021dr3}.
Similarly, some recent methods \citep{shao2020self, pohlen2018observe} include an explicit penalty which discourages gradient updates from affecting the target values. 

\section{Learning dynamics and generalization}\label{sec:theory}
This section will explore a tension between learning dynamics in neural networks, which tend to `generalize-then-memorize' \citep{kalimeris2019sgd}, and the dynamics of temporal difference learning with tabular value functions, discussed in Section~\ref{sec:vf_gen}, which tend to pick up information about the value function's global structure only late in training.
We go on to study how these learning dynamics may affect the structure of gradient updates in the function approximation setting in Section~\ref{sec:fa_gen}.

\textbf{Eigendecomposition of transition operators.}
An important concept in our theoretical analysis will be that of the eigendecomposition of the environment transition matrix. We will follow the precedent of prior work in considering diagonalizable transition matrices, for which further discussion can be found in many excellent prior works  \citep{machado17a, stachenfeld2017hippocampus, mahadevan2005proto}.
\begin{figure*}
    \centering
    \includegraphics[width=0.95\linewidth]{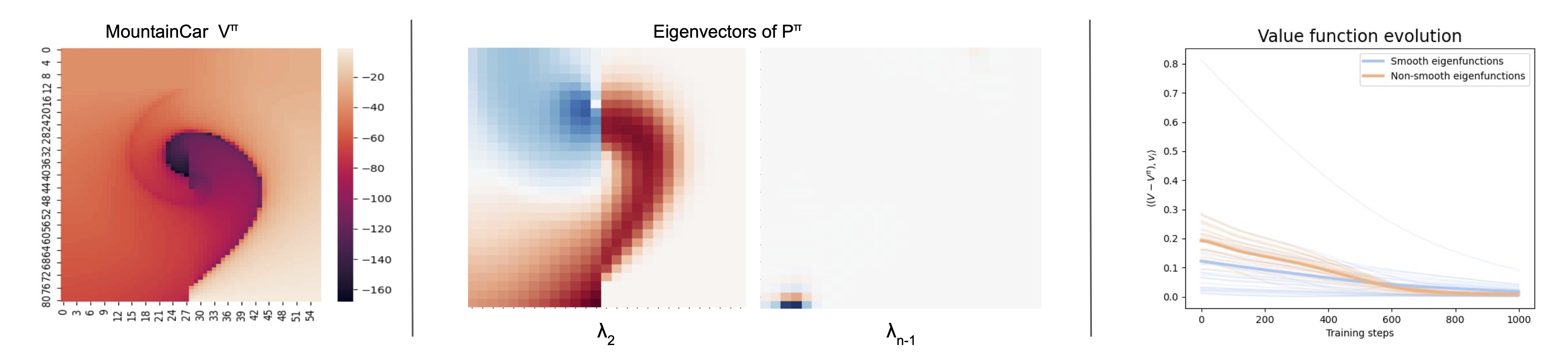}
    \vspace{-0.5em}
    \caption{Left: value function of a near-optimal policy on MountainCar. States correspond to velocity (x-axis) and position (y-axis). Middle: eigenvectors associated with this policy computed for a discretization of the MountainCar state space. Right: value approximation error by eigenbasis coefficient along a trajectory generated by tabular TD updates with learning rate $\alpha = 0.1$ on the discretized MountainCar MDP. We compare 25 of the most-smooth eigenfunctions with 25 eigenfunctions corresponding to negative eigenvalues, and normalize the error by the magnitude of the projection of $V^\pi$ onto the basis spanned by each set of vectors. Each transparent line corresponds to the dot product with a different eigenvector, while the solid lines show the mean over each subspace. }
    \label{fig:mdp}
\end{figure*}
The relationship between the smoothness of an eigenfunction and its corresponding value has been noted in prior work \citep{mahadevan2005proto,mahadevan2007proto}. However, previous discussion of this connection has defaulted to an intuitive notion of smoothness without providing an explicit definition. We provide a concrete definition of the smoothness of a function on the state space $\states$ of an MDP $\mdp$ in order to provide an unambiguous characterization to which we will refer throughout this paper.

\begin{definition}
Given a function $V : \states \rightarrow \mathbb{R}$, MDP $\mdp$, and policy $\pi$, define its expected variation $\rho(V)$ as 
\begin{equation}\rho(V) = \sum_{x \in \states} |V(x) - \mathbb{E}_{P^\pi(x'|x)}V(x')| \; . \end{equation}
We say $V$ is \emph{smooth} if $\rho(V)$ is small.
\end{definition}
This expression reveals a straightforward relationship between the eigenvalue $\lambda_i$ associated with a normalized eigenvector $v_i$ and the smoothness of that eigenvector:
\begin{equation}
    \sum_{x\in \states}|v_i(x) - \mathbb{E}_{P^\pi(x'|x)} v_i(x') | = \sum_{x \in \states} |(1-\lambda_i) v_i(x) | 
\end{equation}
In other words, the eigenvalue of an eigenvector precisely determines the variation of the eigenvector over the entire state space. If $\lambda = 1$, for example, then the eigenvector must be constant over the MDP, whereas if $\lambda = -1$, then we have $\mathbb{E}_{P^\pi(x'|x)}[V(x')] = - V(x)$ and the expected value fluctuates between extremes when stepping from one state to another. The \textit{variance} over next-state values can in principle be large even for functions of low variation by our definition, though in our empirical evaluations (see e.g. Figure~\ref{fig:mc-ff}) smooth eigenvectors tended to also exhibit little variance. For our analysis of the \textit{expected} updates performed by TD learning, we will find the smoothness of the expected updates to be a more useful quantity than the variance.

\subsection{Tabular dynamics}
\label{sec:vf_gen}
We begin by studying the learning dynamics of tabular value functions. We consider a continuous-time approximation of the dynamics followed by the value function using different update rules. Our analysis will contrast Monte Carlo updates, which regress on the value function, with Bellman updates, which regress on bootstrapped targets and correspond to the expected update performed by TD learning. For simplicity, we will ignore the state visitation distribution; analogous resulst for non-uniform state-visitation distributions are straightforward to derive from our findings. We follow the approach of \citet{lyle2021effect} in expressing the dynamics of Monte Carlo (MC) updates as a continuous-time differential equation 
\begin{equation*}
    \partial_t V_t = V^\pi - V_t
    \end{equation*}
where $V_t \in \mathbb{R}^\states$ is a function on the state space $\states$ of the MDP, resulting in the trajectory
    \begin{equation*}
    V_t = \exp ( -t)(V_0 - V^\pi) + V^\pi \, .
    \end{equation*}

Intuitively, this corresponds to a `straight line' trajectory where the estimated value function $V_t$ converges to $V^\pi$ along the shortest path in $\mathbb{R}^{\states}$. In practice, most deep RL algorithms more closely resemble temporal difference updates, which are expressed as 

\begin{align}
    \partial_t V_t &= -(I-\gamma P^\pi)V_t + R^\pi \\
    V_t &= \exp ( -t (I-\gamma P^\pi))(V_0 - V^\pi) + V^\pi \, . \label{eq:td_dynamics}
\end{align}

Whereas under Monte Carlo learning the value function converges equally quickly in all dimensions, its convergence under temporal difference learning depends on the environment transition matrix. We will consider the decomposition of a function $V : \states \rightarrow \mathbb{R}$ as a sum of the basis of eigenvectors of the MDP transition operator $P^\pi$, written $\{v_1, \ldots, v_{|\states|} \}$, obtaining $V = \sum_{i=1}^{|\states|}\alpha_i v_i$ for some (unique) set of coefficients $(\alpha_i)_{i=1}^{|\states|}$.
Under this decomposition, we can show that a predicted value function trained via TD learning will converge more slowly along smooth eigenvectors of $P^\pi$.

\begin{restatable}{obs}{convergence}\label{obs:convergence}
    Let $P^\pi$ be real diagonalizable, with eigenvectors $v_1, \dots, v_{|\states|}$ corresponding to eigenvalues $\lambda_1 > \dots \geq \lambda_{|\states|}$, and let $V_t$ be defined as in Equation~\ref{eq:td_dynamics}. Write $V_t = \sum_{i=1}^{|\states|} \alpha^t_i v_i$ to express the value function at time $t$ with respect to the eigenbasis $\{v_i\}$. Then the convergence of $V_t$ to the value function $V^\pi = \sum_{i=1}^{|\states|} \alpha^\pi_i v_i$ can be expressed as follows:
    \begin{align*}
       \alpha^t_i - \alpha^\pi_i &=  \exp(-t(1-\gamma \lambda_i)) (\alpha_i^0 - \alpha_i^\pi) \, .
    \end{align*}
\end{restatable}

The implications of Observation~\ref{obs:convergence} on the learned value function depend to some extent on the eigendecomposition of $V^\pi$. If $V^\pi$ is equal to the constant function, then we expect the high-frequency components of $V_t$ to quickly converge to zero. If $V^\pi$ puts weight on non-smooth eigenvectors, then early values of $V_t$ may assign disproportionately high weights to these components relative to their contribution to $V^\pi$. In practice, value functions tend to exhibit a mixture of smooth and discontinuous regions, as can be seen in the illustration of a near-optimal value function in MountainCar in Figure~\ref{fig:mdp}.
The corresponding expression of $V^*$ with respect to the eigenbasis of $P^\pi$ consequently places non-zero coefficients on eigenvectors corresponding to negative eigenvalues in order to fit this discontinuity, though its spectrum is dominated by smooth eigenfunctions. 
The following result highlights that non-smooth components of a predicted value function, while contributing relatively little to the Monte Carlo error, contribute disproportionately to the TD error, providing an incentive to fit these components early in training.
\begin{figure*}
    \centering
    \includegraphics[width=0.55\linewidth]{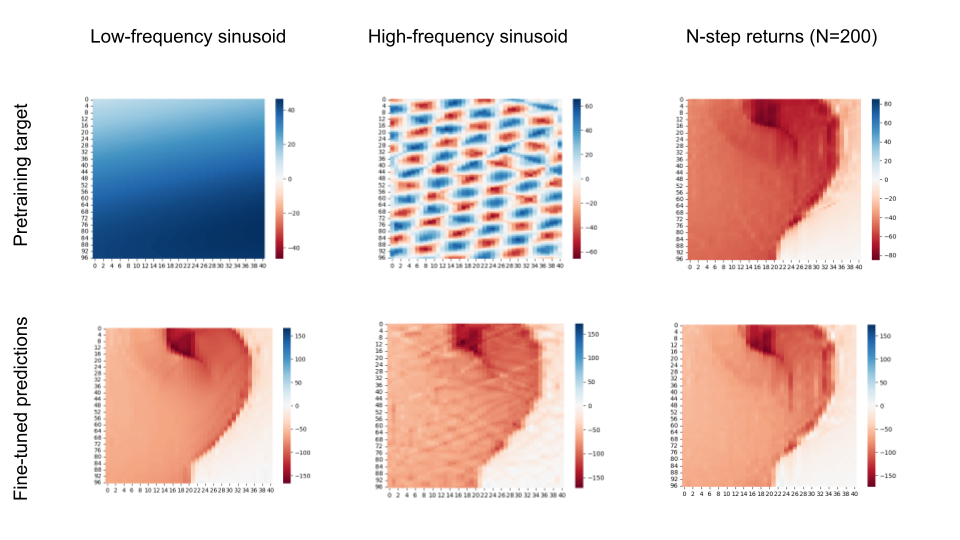}
    \includegraphics[width=0.4\linewidth]{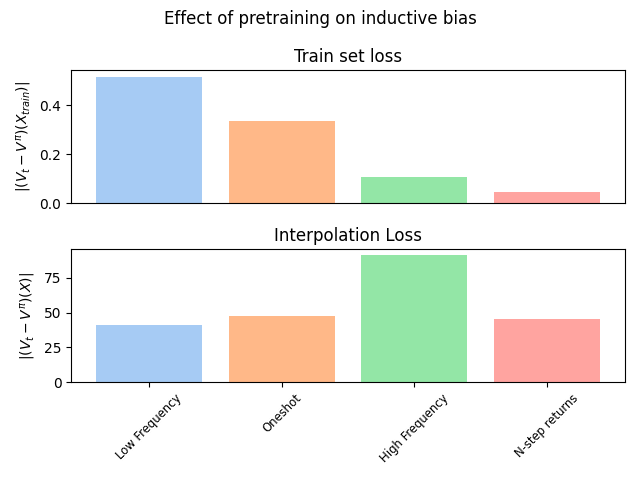}
    \caption{Networks trained to fit high-frequency target functions exhibit pathological interpolation properties when later fine-tuned on a value function. Left: visualization of pre-training targets (top) and final value estimate (bottom) after the pre-trained network is fine-tuned on the value function. Right: loss on the set of training states (top) and a finer-grained set of states which interpolate the training set (bottom) of each fine-tuned network.}
    
    \label{fig:mc-ff}
\end{figure*}

\begin{restatable}{thm}{tderror}
Let $P^\pi$ be real diagonalizable with eigenvalues $\lambda_1 > \dots > \lambda_n$ and $(v_k)_{k=1}^n$ the corresponding (normalized) eigenvectors. Then for any value function $V$, the TD error $\TD(V_t) = \|V_t - T^{\pi} V_t\|^2$ can be bounded as as
\begin{align}
    \|\TD(V_t) \|^2 &= \| T^\pi V_t - V_t \|^2\\
    &= \|\sum (1-\gamma \lambda_i)(\alpha^\pi_i -\alpha^t_i)(v_i)\| ^2  \\
    &\leq \sum_{i=1}^n (\alpha^\pi_i - \alpha^t_i)^2 (1-\gamma \lambda_i)^2  \; 
\end{align}
with equality when $P^\pi$ has orthogonal eigenvectors.
\end{restatable}
Monte Carlo updates, which simply regress on the value function, give equal weight to errors along any component of the basis. These incentives provide some intuition for the different trajectories followed by Monte Carlo and TD updates: in order to minimize the TD loss, the value function must quickly become accurate along non-smooth components of the value function; however, its error due to smooth components such as the bias term of the function will have little effect on the loss and so converges more slowly. We provide an illustrative example of the relationship between the eigenvalue associated with a subspace and the convergence rate of the value function in that subspace in Figure~\ref{fig:mdp}.

\subsection{Function approximation with kernels}
\label{sec:fa_gen}
\begin{figure*}[h]
\centering
    \includegraphics[width=0.348\linewidth]{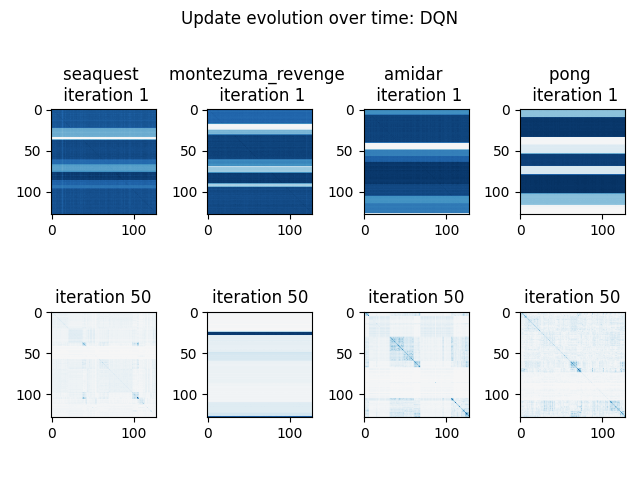}
    \includegraphics[width=0.3\linewidth]{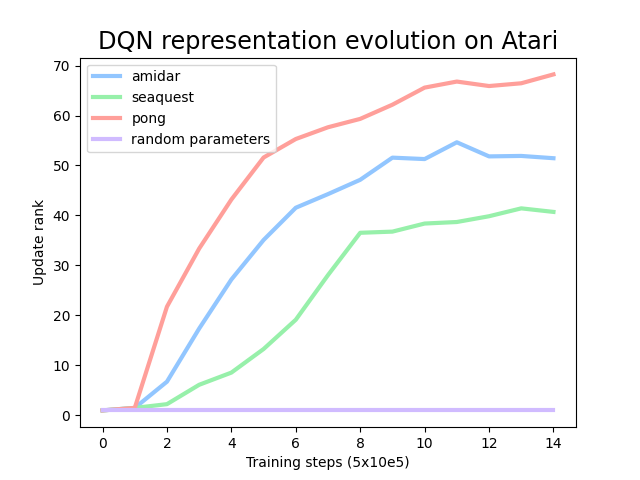}
    \includegraphics[width=0.3\linewidth]{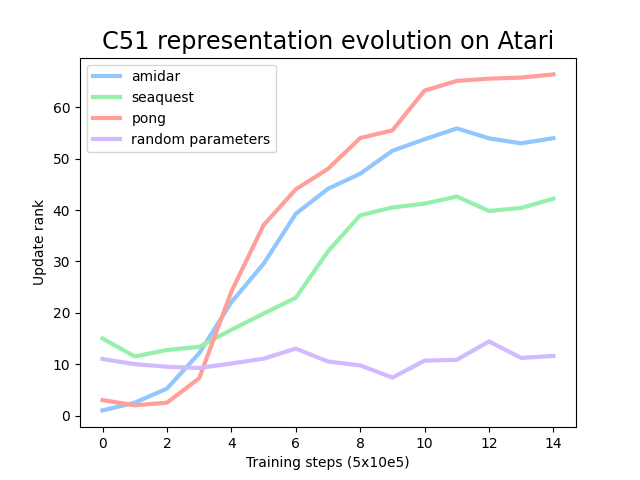}
    \vspace{-0.5em}
    \caption{Agents trained on games from Atari. The networks initially exhibit low update rank, but after 50 iterations (5M frames of experience), the updates rank increases significantly. This is tracked in the bottom plots over the course of approximately 7M frames. Random parameters refers to the update rank obtained by a randomly initialized neural network.}\label{fig:qualitative}
\end{figure*}
Most function approximation schemes leverage the assumption that states which are close together in observation space are likely to have similar values; i.e. they encode a preference towards smooth (with respect to the observations) functions. This pushes against the tendency of temporal difference updates to encourage $V_t$ to fit the components of the value function with large variation first.
To investigate this tension, we consider the \textit{kernel gradient descent} regime.

Formally, a kernel is a positive definite symmetric function $K : \mathcal{X} \times \mathcal{X} \rightarrow \mathbb{R}$. In our case, we will define $\mathcal{X}$ to be the state space of an MDP. Letting $\mathbf{x} \subseteq \mathcal{X}$, we denote by $\tilde{K}$ the (symmetric) matrix $K(\mathbf{x},\mathbf{x})$ with entries $K(\mathbf{x},\mathbf{x})_{i,j} = K(\mathbf{x}_i, \mathbf{x}_j)$. Loosely speaking, a kernel encodes the similarity between two states, allowing us to incorporate certain forms of inductive bias into the value function approximation. Importantly, the similarity of two states under $K$ does not inform us about how similar the states' initial values are, but rather how an update to the value function at one state influences the value of the other; in other words, in our setting it is a proxy for the \textit{interference} between two states.
Under kernel gradient descent, the trajectory of a function is defined in terms of a kernel $K$ and the function-space gradient of a cost function. We can translate TD semi-gradient updates into the kernel gradient descent regime as follows:
\begin{equation}
    \partial_t V_t = \tilde{K} ((\gamma P^\pi - I)V_t + R^\pi) \, .
\end{equation}
It is straightforward then to obtain analogous results as before on the convergence of $V_t$ to $V^\pi$ based on the eigendecomposition of the matrix $\tilde{K} (\gamma P^\pi - I)$ in cases where this matrix is positive definite, though many notable cases occur where this condition does not hold. This decomposition will not in general have a closed form in terms of the eigendecompositions of $\tilde{K}$ and $P^\pi$, but special cases have been studied in the setting of linear regression by \citet{ghosh2020representations} and can be related to kernel gradient descent straightforwardly as discussed in Appendix~\ref{apx:proofs}.
This setting also describes the dynamics of neural networks in the limit of infinite width \citep{jacot2021neural, fort2020deep, lee2020finite}, which follow kernel gradient descent with respect to the neural tangent kernel.

A more interesting case occurs when we assume some states in the environment are not updated during training. 
\begin{restatable}{thm}{ntk}\label{thm:ntk}
Let $K$ be a kernel and $\pi$ a fixed policy in an MDP with finite state space $X$. Let $\Xtrain \subset \states$ be a subset of states in the support of $\pi$, $\Xtest = \states \setminus \Xtrain$, and let $V_t$ be a value trajectory obtained by applying kernel semi-gradient updates on the set $\Xtrain$ to some initial value function $V_0(\Xtrain)$ with kernel $K$. Let $K_{\mathrm{all}}$ be defined as
\begin{equation}
    K_{\mathrm{all}} = K(\Xtrain, \Xtrain) \oplus K(\Xtest, \Xtrain) \, .
\end{equation}Then the trajectory of $V_t$ on the entire state space $X$ will be as follows,
    \begin{align}
        \partial_t V_t(X) &= (K_{\mathrm{all}})  [ (T^\pi V_t - V_t) (\Xtrain)]\;.
    \end{align}
\end{restatable}
A full derivation is provided in Appendix~\ref{apx:proofs}. These dynamics diverge notably from the standard kernel gradient descent regime in that changes to predictions on the test set can now influence the dynamics of $V_t$ on the training set. A large $K(\Xtest, \Xtrain)$ implies that updates to the training set carry great influence over predictions on the test set, but at the cost of increasing asymmetry in $K_{\mathrm{all}}$ when viewed as an operator on $\mathbb{R}^\mathcal{X}$. In Appendix~\ref{appx:kernel-gd} we illustrate how this asymmetry can harm stability in the case of a simple radial basis function kernel when the test states are used as bootstrap targets. 
Combining insights from \ref{thm:ntk} and Observation~\ref{obs:convergence}, we arrive at an intriguing conclusion: in the case of smooth kernels, the components of the value function most suitable to approximation via the kernel $K$ are precisely those which appear in the value estimate of the training set only later in the trajectory. As a result, the kernel does not receive the necessary information to generalize accurately to new observations. This observation runs contrary to the standard kernel regression regime, where one argument in support of early stopping is that kernel gradient descent methods converge along smooth components fastest \citep{jacot2021neural}. At the same time it is an obvious effect of bootstrapping, which requires that the agent update its predictions several times in order to propagate information about the value function through the entire input space. This effect is illustrated in Figure~\ref{fig:kernel-generalization} in Appendix~\ref{appx:kernel-gd}.

\subsection{Non-linear function approximation}\label{sec:non-linear}
The linearized dynamics followed in the neural tangent regime fail to take into account the evolution of the network \textit{features}. We now turn our attention toward the effect of temporal difference updates on the gradient structure of a function approximator by considering the second-order effects of TD semi-gradient updates under finite step sizes. We consider the system
    \begin{equation} \label{eq:discrete_dynamics}
        \theta_{t+1} \gets \theta_t + \alpha \nabla_\theta V(\theta_t) \cdot [(\gamma P^\pi - I)V(\theta_t) + R^\pi ] 
    \end{equation} 
which can be viewed as an Euler discretization of the dynamics described in Equation~\ref{eq:td_dynamics} in the presence of function approximation.
We will use the notation $f(\theta_t)$ to refer to the semi-gradient update on parameters $\theta_t$ inducing value function $V_{\theta_t}$, and write $\TD(\theta) = \frac{1}{2}\| V_\theta - \square T^\pi V_\theta \|^2$, where the $\square$ denotes a stop-gradient. This results in the following gradient flow:
\begin{equation}\label{eq:gradient-flow-theta}
    \partial_t \theta_t = \nabla_\theta V(\theta_t) \cdot [(\gamma P^\pi - I)V(\theta_t) + R^\pi ]  \; .
\end{equation}

Using the continuous-time system \eqref{eq:gradient-flow-theta}
to approximate the discrete-time system \eqref{eq:discrete_dynamics} will gradually accumulate increasing errors, proportional to $(\alpha n)^2$, as it does not take into account the effect of the discrete step size on higher-order gradients of $V_\theta$. We apply a similar analysis to that of \citet{barrett2021implicit} and \citet{ smith2020origin} to understand the effect of the discrete learning dynamics on the gradient structure of $V_\theta$ itself. We let
\begin{equation} \small \label{eq:second-correction}
    f_1(\theta) = \frac{1}{2} \nabla_\theta \| \nabla_\theta \TD(\theta) \|^2 + \gamma (\nabla_\theta ^\top V P^\pi \nabla_\theta V) f(\theta)
\end{equation}
to obtain a second-order correction describing the effect of gradient descent on the gradient structure of the learned representation.
\begin{restatable}[Second-order dynamics]{obs}{theoremsecond} \label{thm:second}
    Let $\theta_t$ be defined by the discrete-time system (\ref{eq:discrete_dynamics}) with step size $\alpha$. Let $f_1 (\theta)$ be defined as in (\ref{eq:second-correction}). Let $\ttheta_t$ denote the trajectory obtained by following the dynamics:
    \begin{align}
      \partial_t  \ttheta_t = f(\ttheta_t) + \frac{\alpha}{2} f_1(\ttheta_t)\;.
    \end{align}
    Then we have $ \theta_{n} \approx \ttheta_{n\alpha} + O( (n\alpha) ^3)$, where $\ttheta_{n\alpha}$ denotes the value of $\ttheta_t$ at time $t=n\alpha$.
\end{restatable}
\begin{figure*}[h]
    \centering
    \includegraphics[width=\linewidth]{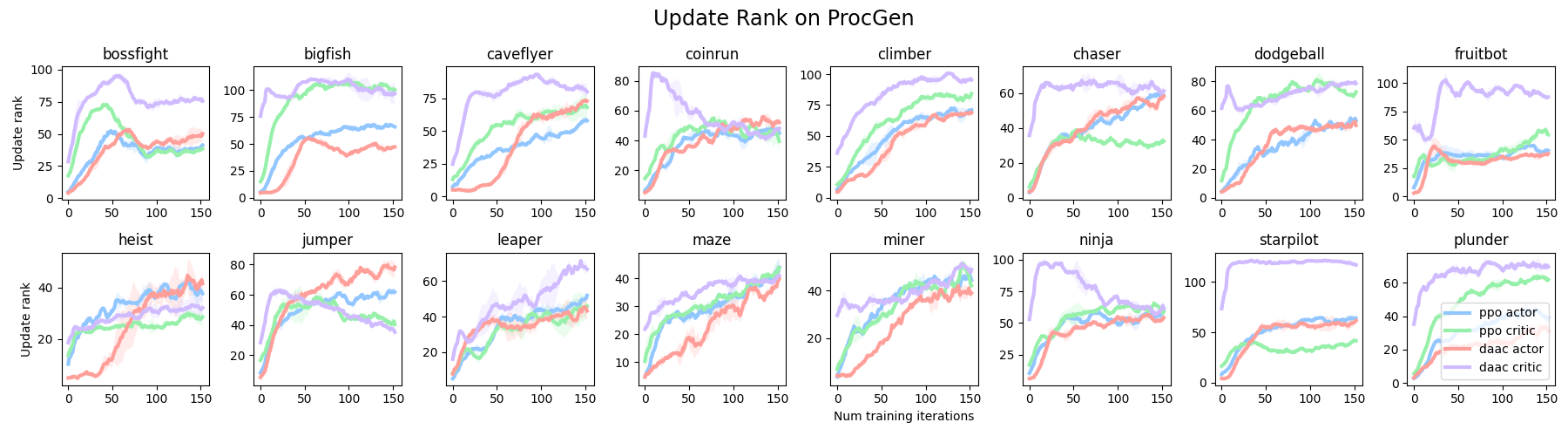} \vspace{-2em}
    \caption{Update dimension of actor-critic methods in ProcGen. Shading indicates minimum and maximum values over 4 seeds. We observe that the update dimension of the separate critic architecture in DAAC (the lilac line) consistently has the highest update rank early in training, while the actors have the lowest rank in the early training stages and only surpass the DAAC critic later in training.}
    \label{fig:daac}
\end{figure*}
The form of $f_1$ constructed in Equation~\ref{eq:second-correction} consists of two terms: a semi-gradient norm penalty term with respect to the instantaneous TD error, and a gradient dot product term which discourages negative interference between nearby states. Since early in training the TD error will tend to be relatively discontinuous and less structured than the true value function (see e.g. Figure~\ref{fig:mc-ff}), the gradient norm penalty will have the effect of discouraging interference between states. Figure~\ref{fig:qualitative} illustrates how fitting highly variable targets early in training can discourage a neural network from smoothly generalizing between states. We observe a similar phenomenon in deep RL agents in Figure~\ref{fig:qualitative}, where networks trained in dense-reward games (whose early TD targets will exhibit greater variation) exhibit weaker interference after training than they did at initialization.

In combination, the findings of this section suggest that the dynamics of temporal difference learning work to discourage interference between states in deep RL by fitting high-frequency components of the value function early in training while also encouraging flatness in parameter space. While this may result in more stable learning, as highlighted in Theorem~\ref{thm:ntk}, it has the double-edged effect of reducing the degree to which the network can generalize to novel observations. The following section will leverage these results to gain insight into deep RL agents trained in rich-observation environments. 

\section{Generalization and interference in deep RL}
\label{sec:rank-exps}
We now explore how TD learning dynamics influence the representations and learned update structure of deep RL agents.
We begin by presenting a quantitative approach to measure the degree to which interference is occurring between states in the agent's visitation distribution. Armed with this metric, we evaluate the following hypotheses. First, that deep neural networks trained with TD updates will exhibit weaker interference between states as training progresses compared to their value at initialization (\textbf{H1}). Second, we conjecture that networks trained with TD learning will exhibit weaker interference than those trained with policy gradient objectives (\textbf{H2}).

\subsection{Representation evolution in value-based RL}

We begin by developing intuition into how the representations learned by deep RL agents evolve over time. Given a set of transitions $\tau_1, \dots, \tau_n$ of the form $\tau_i=(x_i, a_i, r_i, x'_i)$ and a value function $V$ with parameters $\theta$, we let $\theta_i$ denote the network parameters after performing an optimization step with respect to the transition $\tau_i$. We then construct a matrix $\updatematrix$ entry-wise as follows:
\begin{equation}
    \updatematrix_{i,j} = V_{\theta_j}(x_i) - V_\theta(x_i) \, .
\end{equation}
See Figure~\ref{fig:qualitative} for an illustration. The properties of this matrix will depend on the optimizer used to perform updates, leading to notable differences from empirical NTK matrix given by gradient dot products under the current parameters \citep{jacot2021neural}. In the case of stochastic gradient descent, the matrix of gradient dot products will yield a first order approximation of the update matrix $A$.
However, many popular optimizers leverage momentum terms and adaptive step sizes, resulting in updates that diverge significantly from the dot product between gradients. In order to account for this effect, we consider the change in the loss or network output directly. 

At one extreme, the update matrix $A$ for a tabular value function will have non-zero entries only along the diagonal and the matrix will have full rank. At the other, if the value function were represented by a single parameter $\theta \in \mathbb{R}$, then every row will be identical and the matrix will have rank one. Thus, the rank of this matrix can be interpreted as a proxy for whether an agent tends to \textit{generalize} updates between states (low rank), or whether it \textit{memorizes} the value of each state-action pair independently from other states (high rank). In our evaluations, we use an approximate version of the rank that discards negligible components of the matrix based on the singular value decomposition, described in more detail in Appendix~\ref{apx:details}. We will refer to this quantity as the \textit{update rank}. An alternative approach outlined by \citet{daneshmand2021batch} involves computing the Frobenius norm of the difference between the matrix $A$ and the identity, however this may overestimate interference in optimizers which use momentum due to constant terms in the update matrix. 

We proceed to evaluate \textbf{H1} by measuring the update rank of deep RL agents trained on popular benchmarks. We train a standard deep Q-network (DQN) architecture on environments from the Atari 2600 suite, and save checkpoints every 10 million frames. We begin by visualizing the evolution of agents' update matrices over the course of training in Figure~\ref{fig:qualitative}. RL agents trained in dense-reward environments tend to develop update matrices which resemble those of tabular value functions. Those trained in the absence of reward, i.e. those for which the target value function has no high-frequency components, maintain low-rank update matrices through training as our theory would predict. We find that similar results hold for a range of update rules, including distributional updates performed in the C51 algorithm \citep{bellemare2017distributional}. We include further evaluations in Appendix~\ref{apx:more-results}.

\subsection{Actor-critic methods}

Policy gradient methods present an opportunity to avoid the pathologies discussed previously in temporal difference targets while still preserving other properties of the RL problem. While these mehtods tend to exhibit other pathologies, in particular suffering from high variance, there is no reason to expect a priori that this variance will discourage interference in the same way as in TD updates. We investigate \textbf{H2} using two different algorithms on the ProcGen suite: PPO \citep{schulman2017proximal}, which uses a shared representation network for both the actor and critic, and DAAC \citep{raileanu2021decoupling}, where there are no shared parameters between the actor and the critic. This setup allows us to study both the effect of the TD loss on a network's update dimension, and long-term effect of TD gradients on the representation. We run our evaluations in the ProcGen environment \citep{cobbe2019quantifying}, which consists of 16 games with procedurally generated levels. While the underlying mechanics of each game remain constant across the different levels, the layout of the environment may vary. The agent is given access to a limited subset of the levels during training, in this case 10, and then evaluated on the full distribution. We will discuss generalization to new levels in the following section; our initial analysis will focus on generalization between observations in the training environments. 

We evaluate the update dimension of the actor and critic networks of each method in Figure~\ref{fig:daac}. Omitting the critic gradients from the actor's representation leads to significantly lower update dimensions early in training in a number of environments, including bigfish, heist, and miner. Further, the critic network in DAAC, which receives only TD gradients, exhibits markedly higher update rank in all environments in at least the early stages of training, and often throughout the entire trajectory, than the other networks which have access to the actor gradients. 
\begin{figure}
    \centering
    \includegraphics[width=0.939\linewidth]{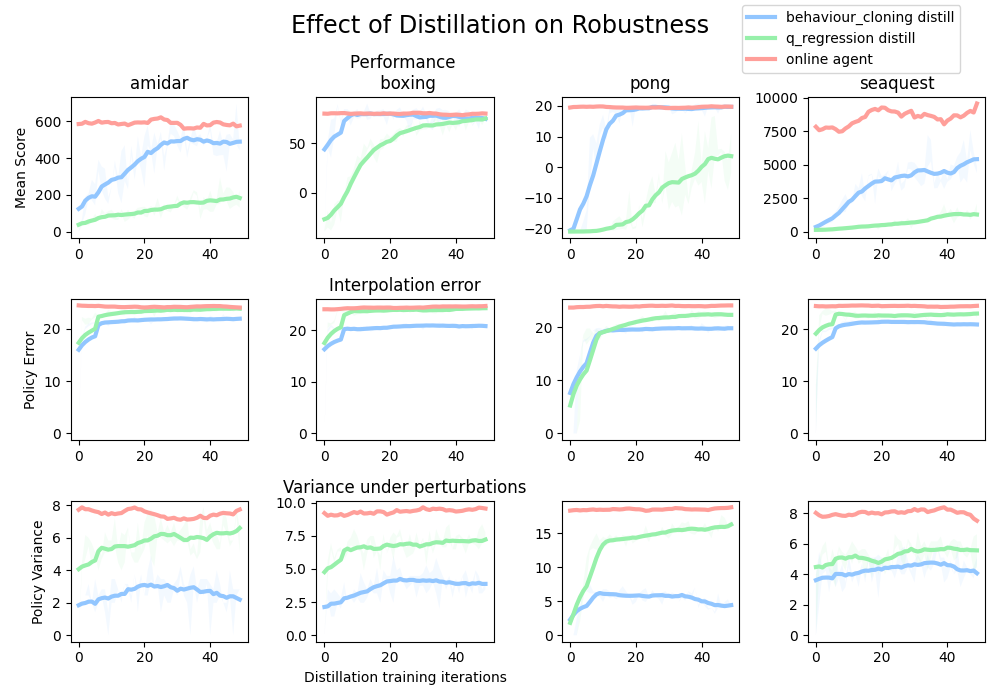}
    \caption{Performance and robustness to perturbations of different distillation approaches in games from the Atari suite. Post-training distillation results in policies that are more consistent under perturbations and under interpolation between observations. Axes indicate the $\ell_1$ norm between the policy on the original input batch and on the perturbed input batch.} \vspace{-1em}
    \label{fig:atari-gen}
\end{figure}

\section{Post-training distillation and generalization}
The previous sections have shown that TD learning dynamics discourage interference, and that while this may have a beneficial effect on stability during training, it can reduce the ability of the network to generalize to new observations. This bias towards memorization arises when, during the network's crucial early development stage, it is trained to fit target functions that do not capture the global structure of the value function. One simple solution to this problem is to train a freshly initialized network on the final value function obtained by TD learning. If the learned value function was able to pick up on the global structure of the value function, then the freshly initialized network will be able to benefit from incorporating this structure into its predictions more systematically than the teacher. Similar approaches have seen success in improving generalization \citep{igl2019generalization} and allowing agents to break through plateaus \citep{fedus2020catastrophic}. The goal of this section is to provide a deeper empirical study into how post-training distillation influences the interference and generalization of the final learned value function.
\begin{figure*}[h]
\centering
    \includegraphics[width=0.97\linewidth]{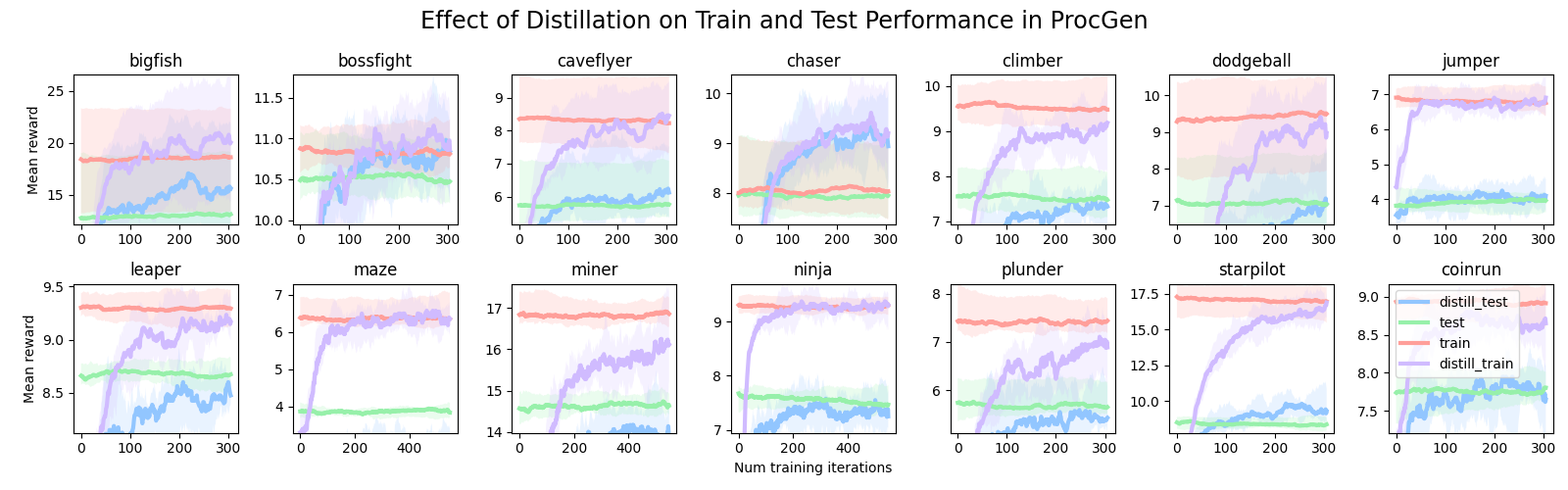}\vspace{-0.5em}
    \caption{Effect of policy distillation on generalization in environments from the Procgen suite. We plot the pretrained networks train environment and test environment performance, along with the performance of the distilled agent on the test environments. We see significant improvement on test environments in bigfish, caveflyer, chaser, climber, and bossfight. }
    \label{fig:procgen-distill}
\end{figure*}
\subsection{Value distillation}
We first consider value distillation as a means of eliminating the counterproductive bias towards memorization induced by early TD targets. We leverage a data collection policy from a pre-trained teacher network $q_t$, and perform distillation of a freshly initialized network $q_s$ on this data. We follow a similar procedure to that of \citet{ostrovski2021the} to perform distillation of the function $q_{\mathrm{s}}$ on data collected sampled from the teacher's replay buffer $\mathcal{B}_T$, leveraging their insight that distillation on \textit{all} action values, rather than only the value of the action taken by the teacher agent, yields significantly higher performance.  We additionally study the effect of behaviour cloning with entropy regularization, obtaining the objectives
\begin{align}\label{eq:value_distill}
    \ell_{\mathrm{VD}}(q_{\mathrm{S}}, q_{\mathrm{T}}) &= \mathbb{E}_{s \sim \mathcal{B}_{\mathrm{T}}} \bigg [\sum_{a \in \mathcal{A}} (q_{\mathrm{S}}(a) - q_{\mathrm{T}}(a) )^2 \bigg ]\\
    \ell_{\mathrm{BC} }(\theta) &= \mathbb{E}_{s,a\sim \mathcal{B}_{\mathrm{T}}}[ \log \pi_\theta(s,a) + \lambda H(\pi_\theta(s)] \label{eq:policy_distill}
\end{align}
where $H(\cdot)$ denotes the entropy of the policy. We set $\lambda = 1e-2$ in our evaluations.
We show results for value distillation \eqref{eq:value_distill}, which regresses on the outputs of the frozen Q-network, and behaviour cloning \eqref{eq:policy_distill}, which predicts the action taken by the frozen Q-network.
We track three quantities: the performance of the learned policy, the robustness of the learned policy to perturbations, and the consistency of the learned policy when interpolating between observations. The performance is measured by following an $\epsilon$-greedy policy in the training environment, with $\epsilon=0.01$. The robustness to perturbations is measured by tracking whether the network takes the same action under a Gaussian perturbation to its input as in the unperturbed observation. Finally, we investigate the network's interpolation behaviour by evaluating whether, given a convex combination of observations $o_1$ and $o_2$, the network takes the same action under the combination as it does in either of the original observations.

Figure~\ref{fig:atari-gen} shows that the distilled networks are more robust to perturbations and are more consistent under interpolations between observations. We observe that the behaviour cloning method matches or nearly matches the performance of the pretrained agent in three of the four environments, while also obtaining the best robustness. Both behaviour cloning and value distillation improve upon the teacher network that was trained online. We conclude that while value distillation can mitigate some of the effect of TD methods on interference, policy-based methods exhibit better robustness properties. This finding motivates the next section, where we will dig deeper into policy distillation.

\subsection{Policy distillation}
This paper has been predominantly concerned with interference, a phenomenon defined on states the agent has encountered during training. Given the success of policy distillation to improve robustness, we now evaluate whether distillation may also improve generalization to novel environments. We return to the ProcGen benchmark, with the hypothesis that post-training distillation of PPO agents should produce policies which improve on the ability of the final trained actor to generalize to new levels, provided that the levels generated for the test set are sufficiently similar to the training distribution that such generalization is feasible. 
We use the same PPO agents as in Figure~\ref{fig:daac} as teacher agents in our evaluations. We train a freshly initialized network (the distillation agent) to minimize the KL divergence with the teacher's policy $\pi_T$ on transitions collected by the teacher and stored in a buffer $\mathcal{B}_T$. Concretely, we minimize a stochastic estimate of the following objective:
\begin{equation}
  \mathbb{E}_{s \sim \mathcal{B}_{T}} [ D_{KL}(\pi_S(s) ||\pi_T(s)) + \lambda H(\pi_S)] \; .
\end{equation}
We then evaluate the distilled agent's performance on the test environments. Results are shown in Figure~\ref{fig:procgen-distill}.

We find that post-training distillation consistently meets or improves upon the generalization \textit{gap} obtained by the original network, in many environments significantly improving on the final test set performance of the PPO agent. We attribute this improvement to the absence of TD learning gradients in the distillation process and the stationarity of the distillation targets, avoiding the pitfalls of non-stationarity highlighted by \citet{igl2020transient}. It is likely that the raw performance obtained obtained by the student could be improved using lessons from the policy distillation literature \citep{czarnecki2019distilling, rusu2016policy, teh2017distral}.
Given these limitations, post-training distillation stands out all the more as a simple and effective way to improve generalization performance. It is further notable that the effect of distillation varies significantly between environments, suggesting that the relationship between interference and out-of-environment generalization depends heavily on the structure of the environment-generating process.

\section{Conclusion}

Our analysis has shown that temporal difference learning targets converge along non-smooth components of the value function first, resulting in a bias towards memorization when deep neural networks are employed as function approximators in value-based RL. In the context of prior work demonstrating that weaker generalization can improve the stability and convergence of RL algorithms, this phenomenon may be beneficial to an agent's stability, but comes at the cost of observational overfitting. We further show that post-training distillation improves generalization and robustness, mitigating some of the tendency of value-based RL objectives to encourage overfitting.
This work presents a crucial first step towards the principled development of robust and stable deep RL algorithms which are nonetheless capable of strong generalization performance. Our insights may prove useful in a range of future directions, such as using different architectures during training and distillation, leveraging larger neural network function approximators to minimize harmful interference, and modifying the update rule used in TD learning to adaptively promote or inhibit interference between inputs. Further, the role of the optimiser is fundamental to the phenomena studied in this paper, and RL-specific optimization approaches may benefit from our findings.
\clearpage
\section*{Acknowledgements}
MK received funding from the ERC under the European Union’s Horizon 2020 research and innovation programme (FUN2MODEL, grant agreement No. 834115). CL is funded by the Open Philanthropy AI Fellowship. Thanks to Georg Ostrovski for provided valuable feedback on an early draft of this paper, along with Maximilian Igl and Charline Le Lan for helpful discussions.

\bibliographystyle{plainnat}
\bibliography{references}

\newpage
\appendix
\onecolumn

\input{appendix}

\end{document}

%% file: appendix.tex
\section{Proofs}
\subsection{Characterizing smoothness in MDPs}
\label{apx:smooth}
Throughout the text, many references are made to 'smooth' functions, without giving a strict definition. While this is useful to convey a rough idea of the types of functions we are interested in, we provide a more rigorous discussion in this section. First, we distinguish between smoothness with respect to a notion of distance in the observation space, for example $\ell_2$ distance between vectors, and distance with respect to the MDP's transition dynamics, which measures how easily the agent can reach one state from another. In most settings of interest, the two definitions will largely agree, motivating our use of the generic term smoothness in our discussion of neural network function approximation. In these cases, the inductive bias of the neural network towards smooth functions over the observation space corresponds to an inductive bias twoards functions that are smooth with respect to the MDP transition dynamics. However, this may not always be the case. For example, when walking through a door that leads from one level to another in a video game; though the last frame of the old level and the first frame of the new one may be visually distinct, they will have low distance in the MDP. 

The notions of smoothness we refer to in Section~\ref{sec:vf_gen} relates to the variation of the value function between adjacent states in time. This definition resembles graph total variation \citep{tovsic2010graph}, which characterizes the degree to which a node's value differs from the average of its neighbours. In our case, we treat the transition matrix $P^\pi$ as a weighted directed graph, and will be interested in the quantity $|V(x) - \mathbb{E}_{P^\pi(x'|x)}[V(x')] |$. We note trivially that if $V$ is an eigenvector of $P^\pi$ with eigenvalue $\lambda$, then
\begin{equation}
    \sum_x |V(x) - \mathbb{E}_{P^\pi(x'|x)} V(x') | = \sum_x |(1-\lambda) V(x) | 
\end{equation}
In other words, the eigenvalue of an eigenvector precisely determines the variation of the eigenvector over the entire state space. If $\lambda = 1$, for example, then the eigenvector must be constant in expectation over the MDP, whereas if $\lambda = -1$, then we have $\mathbb{E}_{P^\pi(x'|x)}[V(x')] = - V(x)$ and the value fluctuates between extremes when stepping from one state to another. We obtain an analogous result if, rather than taking the max over states, we take a weighted average or a sum. 

\subsection{Proofs of main results}
\label{apx:proofs}
\convergence*

\begin{proof}
Recall we assume the following dynamical system

\begin{align*}
    \partial_t V_t &= -(I - \gamma P^\pi)V_t + R 
    \intertext{Inducing the trajectory}
    V_t &= \exp( - t(I - \gamma P^\pi)) (V_0 - V^\pi) + V^\pi 
    \intertext{As we assume $P^\pi$ is diagonalizable, this implies that $(I - \gamma P^\pi)$ is also diagonalizable. Let $u_1, \dots, u_n$ denote the right eigenvectors of $P^\pi$ with corresponding eigenvalues $\lambda_1 \ge \dots \ge \lambda_n$. Let $V_0 = \sum \alpha^0_i u_i$. }
    V_t &= \sum \alpha_i^t u_i \\
    &= \exp (-t(I - \gamma P^\pi)) (\sum \alpha_i^0 - \alpha^\pi_i u_i) + \sum \alpha^\pi_i u_i \\
    &= \sum  \exp(-t(1-\gamma \lambda_i)) \bigg ( \sum  (\alpha_i^0 - \alpha_i^\pi  ) u_i + \sum \alpha_i^\pi u_i \bigg )
    \intertext{Now, we consider the value of $V_t - V^\pi$ along each coordinate. Note that we have not assumed an orthogonal eigenbasis, thus cannot speak directly to the norm of the projection of this difference onto the eigenspace corresponding to each eigenvector $\lambda_k$. However, treating the eigendecomposition as a basis, we can discuss how the coordinates $\alpha^t_i$ of the value function $V_t$ converge with respect to this basis.}
    |V_t - V^\pi|[i] = |\alpha^t_i - \alpha^\pi_i| &=   | \exp(-t(1-\gamma \lambda_i)) (\alpha_i^0 - \alpha_i^\pi) + \alpha_i^\pi  - \alpha^\pi_i| \\
    &=|\exp(-t(1-\gamma \lambda_i))  (\alpha_i^0 - \alpha_i^\pi) | = \exp(-t(1-\gamma \lambda_i)) | (\alpha_i^0 - \alpha_i^\pi) | 
\end{align*}
We conclude by noting that for large values of $\lambda_i$, the exponential term $\exp ( - t(1 - \gamma \lambda_i)) $ will decay more slowly as a function of $t$ than for smaller values of $\lambda_i$. Thus, these coordinates (which correspond to non-smooth functions over the state space) will converge fastest. When the eigenvectors form an orthogonal basis, as is the case for symmetric $P^\pi$, we can go further and observe that this convergence will apply to the norm of the projection of the value function into the corresponding eigenspace.   Thus for symmetric $P^\pi$, we obtain the following stronger convergence result, where $U_k$ denotes the eigenspace corresponding to the eigenvalue $\lambda_k$.
\begin{equation}
    \| \Pi_{U_k} ( V_t - V^\pi ) \| = \exp(-t (1 - \gamma \lambda_k)) \| \Pi_{U_k} (V_0 - V^\pi ) \|
\end{equation}
\end{proof}
\tderror*
Let $V_0 = \sum  \alpha_i v_i$. Then, letting $V_t$ be defined as in Equation~\ref{eq:td_dynamics}.
\begin{equation}
\TD(V_t) \leq \sum_{i=1}^n \exp(-2t(1-\gamma \lambda_i))( \alpha^\pi_i -  \alpha_i^0)^2 (1-\gamma \lambda_i)^2 \; .
\end{equation}
\begin{proof}

By our assumption on the diagonalizability of $P^\pi$, we can leverage the previous result on the coordinates of $V_t$.
\begin{align*}
    V_t - V^\pi &=  \sum  \exp(-t(1-\gamma \lambda_i)) \bigg ( \sum  (\alpha_i^0 - \alpha_i^\pi  ) u_i \bigg ) \\
    \intertext{We then bound the TD error as follows.}
    \|V_t - \gamma P^\pi V_t - R \|^2 &= \| V_t - \gamma P^\pi V_t + \gamma P^\pi V^\pi -\gamma P^\pi V^\pi -R \| \\
    &= \| V_t - \gamma P^\pi V^\pi - R -  \gamma P^\pi (V_t - V^\pi) \| \\
    \intertext{Since $V^\pi = R + \gamma P^\pi V^\pi$, we obtain the following.}
    &= \| (I -  \gamma P^\pi) (V_t - V^\pi )\|^2 \\
    &= \| \sum (1 - \gamma \lambda_k) (\alpha_i^t - \alpha_i^\pi) u_i \|^2 \\
    &\leq \sum (\alpha^\pi_i - \alpha^t_i)^2(1-\gamma \lambda_i)^2
\end{align*}
The remainder follows a straightforward substitution.
\end{proof}

\theoremsecond* 

\begin{proof}
While our prior analysis has considered the continuous time system $\ttheta_t$, this does not perfectly approximate the discrete system $\theta_t$. When a fixed step size is used, the first-order continuous-time approximation accrues error roughly proportional to $\alpha t$. We then follow the procedure of \citet{barrett2021implicit}, applying a Taylor expansion to the evolution of $\ttheta_t$ with respect to time. We will use the notation $\ttheta(t)$ to denote the explicit dependence of $\ttheta$ as a function of time.
\begin{align}
    \ttheta(\alpha t) &= \ttheta(0) + \sum \frac{(\alpha t)^n}{n!} \theta^{(n)}(0) \\
    &= \ttheta(0) + \alpha t f(\ttheta(0)) + \frac{(\alpha t)^2}{2} \nabla_\theta f \cdot f(\ttheta(0)) +O(\alpha^3)\\
    &= \ttheta(0) + \alpha t f(\ttheta(0)) + \frac{(\alpha t)^2}{2} f_1(\theta(0)) +O(\alpha^3) \\
    \intertext{Relating this back to the discrete system $\theta_t$}
    \theta_{1} &= \theta_0 + \alpha f(\theta_0) = \ttheta(0) + \alpha f(\ttheta(0)) \\
    \theta_1 &= \ttheta(1) - \frac{\alpha^2}{2}f_1(\ttheta(0)) + O(\alpha^3)
    \intertext{Thus, the system $\partial_t \check{\theta}_t = f(\check{\theta}_t) + \alpha^2/2 f_1( \check{\theta}_t)$ satisfies}
    \theta_1 &= \check{\theta}(1) + O(\alpha^3)
\end{align}

We begin by observing that $\nabla_\theta \| V_\theta - \square T V_\theta \|^2 = (V_\theta - T V_\theta) \cdot \nabla_\theta V_\theta = f(\theta)$. 
\begin{align}
   \theta_{} &= \theta_0 + \alpha n f(\theta_0) + (\alpha n) ^2/2 \nabla_\theta f (\theta_0) \cdot f(\theta_0) + O( (\alpha n)^3)  \\
   &= \theta_0  + \alpha n f(\theta_0) + \frac{(\alpha n)^2}{2} f_1(\theta_0) + O((n\alpha)^3)\\
   \intertext{We then express $f_1(\theta)$ as follows.}
   f_1(\theta_0) &= \nabla_\theta[ f (\theta_0)] \cdot [f(\theta_0)] \\
   &= [\nabla^2_w V_\theta \cdot ((\gamma P^\pi - I)V_\theta + r) + \nabla_\theta V_\theta \cdot ((\gamma P^\pi - I) \nabla_\theta V_\theta)][f(\theta)] \\
   &= [\nabla_\theta^2 V_\theta \cdot ( (\gamma P^\pi - I)V_\theta + r) + \nabla_\theta V_\theta \cdot \nabla_\theta V_\theta][f(\theta)] + \gamma [\nabla_\theta V_\theta P^\pi \nabla_\theta V_\theta][f(\theta)] \\
   \intertext{Noting that the left hand side term is equal to the gradient of the gradient norm penalty for the stop-gradient version of the TD regression problem, we simplify as follows:}
   &= \frac{1}{2}\nabla_\theta \| \nabla_\theta \frac{1}{2}\|V_\theta - \square T^\pi V_\theta\|^2 \|^2 + \gamma [\nabla_\theta V_\theta \cdot P^\pi \cdot \nabla_\theta V_\theta][f(\theta)]
\end{align}
We note that, unlike in the stochastic gradient descent setting, $f_1$ does not correspond to a gradient of any function. Instead,  it corresponds to the second-order we would get for a frozen target, which corresponds to a gradient norm penalty, plus a term that measures the alignment of the gradients at each state and its expected successor. Intuitively, both of these terms minimize the `variance' in the loss induced by noisy, discrete gradient steps. The flatter loss surfaces induced by the gradient norm penalty will naturally lead to greater robustness to parameter perturbations. The gradient alignment term reflects the observation previously that non-smooth functions contribute the most to the TD error, and so encourages the first-order gradient effects on successive states to move in a similar direction. 

We note that this final observation seems to be at odds with the tendency for TD learning to encourage more tabular updates. Why would a second-order correction term which promotes flat minima and gradient alignment result in tabular updates? To answer this, we point to the tendency of TD targets to converge along. the non-smooth components of the value function first. We are therefore faced with finding a flat region of parameter space to fit a discontinuous function. A representation which succeeds at this will benefit from minimizing interference between states, as the gradients for one transition will be on average uncorrelated with even nearby other states. The gradient alignment penalty suggests that, while the implicit regularization will prefer flat minima, smooth interference patterns which move other states in a similar direction to the current state will be penalized less than non-smooth directions.
\end{proof}

\begin{restatable}{cor}{corr-second}
    The second-order dynamics push features towards precisely the worst direction w.r.t. stability. I.p. looking at the set of positive definite representations introduced by \citet{ghosh2020representations} we see
    \begin{equation}
        \{v : v^\top P^\pi v <  \gamma^{-1} \|v\|_{\Xi} \}
    \end{equation}
    whereas the optimal gradients for the second order term implicitly solve the following optimization problem
    \begin{equation}
        \min  \mathbb{E}_{x \sim \eta(x)}[g(x)^\top g(x) - \gamma g(x)^\top (P^\pi g)(x)]
    \end{equation}
    
\end{restatable}

\ntk* 
\begin{proof}
We leverage the dynamics $\partial_t V_t = K(X,X) \nabla_\theta V_\theta \cdot ((\gamma P^\pi - I) + r)$ and follow the derivation of Section 5 of \citet{jacot2021neural}. 
\end{proof}

We can develop intuitions for the kernel gradient descent setting by considering the special case of linear function approximation, where $K(x_1, x_2) = \langle \phi(x_1), \phi(x_2) \rangle$ for some feature map $\phi$. For the moment, we will define $\Phi$ to be a matrix consisting of features for every state in the state space $X$ (i.e. we update all states in the mdp at once). We then obtain
\begin{align}
    \partial_t \mathbf{w}_t & = \alpha \Phi^\top (R^\pi + \gamma P^\pi \Phi \mathbf{w}_t - \Phi \mathbf{w}_t) \label{eq:w-flow} \, .
\end{align}
We can express the evolution of the value function constructed by multiplication of $\Phi$ and $w$ as follows.
\begin{align}
    \partial_t V_t &= (\partial_w V_t)^\top \partial_t w_t = \Phi \partial_t w_t \\
    &= -\Phi (\Phi^\top (I - \gamma P^\pi) \Phi) w \\
    &= - \Phi \Phi^\top (I - \gamma P^\pi) V_t \\
    &= - K (I - \gamma P^\pi) V_t
    \intertext{We further consider the dynamics of the value function on inputs outside of the set of states on which the Bellman updates are computed as follows.}
    \partial_t V_t(\xtest) &= (\partial_w V_t(\xtest))^\top \partial_t w_t \\
    &= - \phi(\xtest)^\top  \Phi^\top (I - \gamma P^\pi) V_t  \\
    &= - K(\xtest, \Xtrain) K(\Xtrain, \Xtrain)^{-1} \partial_t V_t
\end{align}

We now lift the assumption that all states are updated. In this more general kernel gradient descent setting, we let $K$ be a kernel as before, with $\tilde{K} = K(\Xtrain, \Xtrain)$ and $\kappa_{\xtest} = K(\xtest, \Xtrain)$. We then obtain the following dynamics
\begin{align}
    \partial_t V_t(\xtest) &= \kappa_{\xtest} \tilde{K}^{-1} \partial_t V_t(\Xtrain) \\
    \intertext{In particular, this results in the following trajectory.}
    V_t(\xtest) &= V_0(\xtest) + \kappa_{\xtest} \tilde{K}^{-1} [ V_t(\Xtrain) - V_0(\Xtrain)]
\end{align}

An interesting case study occurs when we consider, e.g., off-policy evaluation where the bootstrap targets used in the TD updates may not have been visited by the agent during training. This will be the case in many offline RL problems, where the action that would be selected by the policy we seek to evaluate was not taken by the behaviour policy, and so the agent leverages bootstrap targets which are not updated directly as part of the training process, but rather only indirectly via the influence of updates to other states. In such cases, we will decompose the state space as $X = \Xtrain \oplus \Xtest$. The dynamics we get in this case look quite different from standard kernel regression, as the dynamics of the training states will depend on the predictions on the `test' states. To condense notation, we will use $T^\pi V_t$ to refer to an application of the Bellman operator $V_t \mapsto \gamma P^\pi V_t + R^\pi$.  

\begin{align}
    \partial_t V_t(\Xtrain) &= \Phi_{\train} \Phi_\train^\top (  (T^\pi V_t)(\Xtrain) - V_t(\Xtrain)) \\
    \partial_t V_t(\Xtest) &= \Phi_{\test} \Phi_\train^\top  ((T^\pi V_t)(\Xtrain) - V_t(\Xtrain))
    \intertext{We note that $(T^\pi V_t)(\Xtrain)$ depends on both $V(\Xtrain)$ and $V(\Xtest)$ due to the application of the Bellman operator $T^\pi$. We thus end up with the following joint system.}
    \partial_t V_t(\Xtrain \oplus \Xtest) &= \Phi_{\test} \Phi_\train^\top  ((T^\pi V_t)(\Xtrain) - V_t(\Xtrain)) \oplus \Phi_{\train} \Phi_\train^\top (  (T^\pi V_t)(\Xtrain) - V_t(\Xtrain)) \\
    \partial_t V_t(\Xtrain \oplus \Xtest) &= (\Phi_{\test} \oplus \Phi_{\train} ) \Phi_\train^\top (  (T^\pi V_t)(\Xtrain) - V_t(\Xtrain)) 
    \intertext{Using a non-standard notation of $K_1 \oplus K_2:= X \mapsto K_1(X) \oplus K_2(X)$, we can then rewrite the above in terms of the dot product kernel $K(x,x')$ as follows.}
    \partial_t V_t(X_{\mathrm{all}}) &= (\tilde{K} \oplus \kappa_{\xtest})  [ (T^\pi V_t - V_t) (\Xtrain)]
\end{align}       
We emphasize that while this at first looks as though the dynamics are independent of the value $V_t(\Xtest)$, this is an artefact of the Bellman operator notation $(T^\pi V_t) (X_t)$, which hides the dependence of the Bellman targets $ T^\pi V_t$ on $\Xtest$. In particular, we can write $(T^\pi V_t)(X_t) = \Pi_{\Xtrain}[\gamma P^\pi V_t (\Xtrain \oplus \Xtest) + R^\pi]$, which makes this dependence more explicit but is less succinct.

\section{Experiment details}
\label{apx:details}
\subsection{Estimation of Update Rank}

To estimate the update rank of an agent, we sample $k$ transitions from the agent's replay buffer and compute the matrix $A(\theta)$ as described in Section~\ref{sec:rank-exps}. We use the agent's current optimizer state and its current parameters in this computation. We then take the singular value decomposition of $A$ to obtain $k$ singular values $S = \{\sigma_1, \dots, \sigma_k\}$. We then threshold using the numerical approach taken in prior works \citep{maddox2020rethinking}, and compute the size of the set $S_{\epsilon} = \{ \sigma \in S : \sigma > \epsilon \max(S) \}$. This allows us to ignore directions of near-zero variation in the update matrix. In practice, we use $\epsilon = 0.1$. 

Because the Q-functions learned by value-based deep RL agents are vector- rather than scalar-valued functions of state, and our estimator depends on an 2-dimensional update matrix, we must make a choice on how to represent the change in the state-value function. We considered taking the maximum over actions, the mean over actions, selecting a fixed action index, and selecting the action taken in the transition on which the update was computed, and found that both choices produced similar trends. In all evaluations in this paper, Q-functions are reduced using the max operator. We apply the same approach for distributional agents by taking the expectation over the distribution associated with each state-action pair. 

To evaluate the policy-based agents, whose outputs correspond to distributions over actions, we compute the norm of the difference in the output probability distributions for each state in lieu of taking the difference of output values. I.e., the entry $A_{i,j} = \| p_\theta(x_j) - p_{\theta_i}(x_j) \|$, where the discrete probability distribution $p_\theta$ is taken as a vector. 
\subsection{ProcGen}
\begin{figure}
    \centering
    \includegraphics[width=0.35\linewidth]{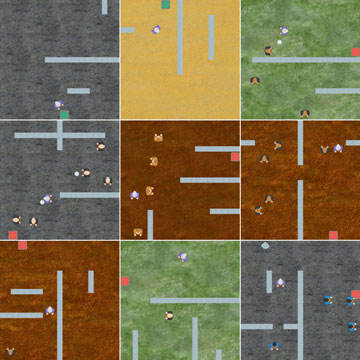}
    \caption{Example levels from the dodgeball environment.}
    \label{fig:procgen-viz}
\end{figure}
The ProcGen benchmark consists of sixteen procedurally generated environments. Each environment consists of a set of randomly generated levels, of which a fixed subset are used for training and a disjoint subset are used for evaluation. Levels differ superficially in their observations and initial sprite layouts but retain the same underlying structure, as can be seen in Figure~\ref{fig:procgen-viz}. The observation space is a box space with the RGB pixels the agent sees in a numpy array of shape (64, 64, 3).

Our PPO and DAAC agents use the same hyperparameters and implementation as is provided by \citet{raileanu2021decoupling}. Our behaviour cloning objective minimizes the KL divergence between the distillation agent the pretrained agent's policies, with an entropy bonus equal to that used to train the original PPO agent.  
\subsection{Atari}

We additionally perform evaluations on environments from the Atari benchmarks. Due to computational constraints, we consider only a subset of the entire benchmark. We obtain a mixture of easy games, such as pong and boxing, and more challenging games like seaquest, where we measure difficulty by the time it takes for the agent to meet human performance. For some experiments, we used the sparse-reward environment Montezuma's Revenge.

In our distillation experiments, we train the original agent for 50M frames using $\epsilon$-greedy exploration with $\epsilon = 0.1$, and train the distillation agents for a number of updates equivalent to 10M frames of data collected online. We base our implementation off of the open-source implementations in \citet{ostrovski2021the}. 

For our behaviour cloning objective, we use the same architecture as is used for DQN, but feed the final layer of actions into a softmax to obtain a probability distribution over actions, which we denote as $P_\theta(a|x)$. Given a state-action pair taken by the target agent, we implement the following behaviour cloning loss for distillation
\begin{equation}
    \ell(\theta, x_i, a_i) = -\log P_\theta(a_i | x_i) -0.1 H(P_\theta(\cdot | x_i))
\end{equation}
where $H$ denotes the entropy of a distribution. We use a replay capacity of 1e6 and allow the pre-trained agent to collect additional data during distillation to further increase the training set size of the distilled agents. 
\section{Additional numerical evaluations}
\label{sec:numerical}
We provide additional numerical evaluations to provide additional insight into the theoretical results of Section~\ref{sec:theory}.
\subsection{Fourier analysis}
We begin by studying the Fourier decomposition of value and reward functions in popular Atari domains by treating the value function as a function of \textit{time} rather than as a function of \textit{observations}. In this sense, the Fourier decomposition is measuring the continuity of the value function with respect to time and so is a closer approximation of the notion of smoothness we focus on in Section~\ref{sec:vf_gen}. We show our evaluations in Figure~\ref{fig:atari_fourier}.
\begin{figure}
    \centering
    \includegraphics[width=18cm]{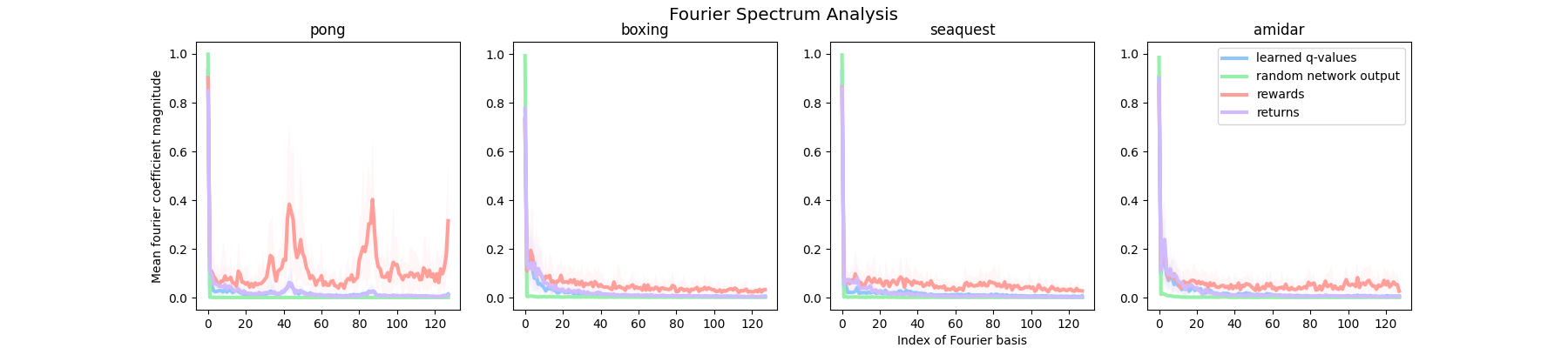}
    \includegraphics[width=18cm]{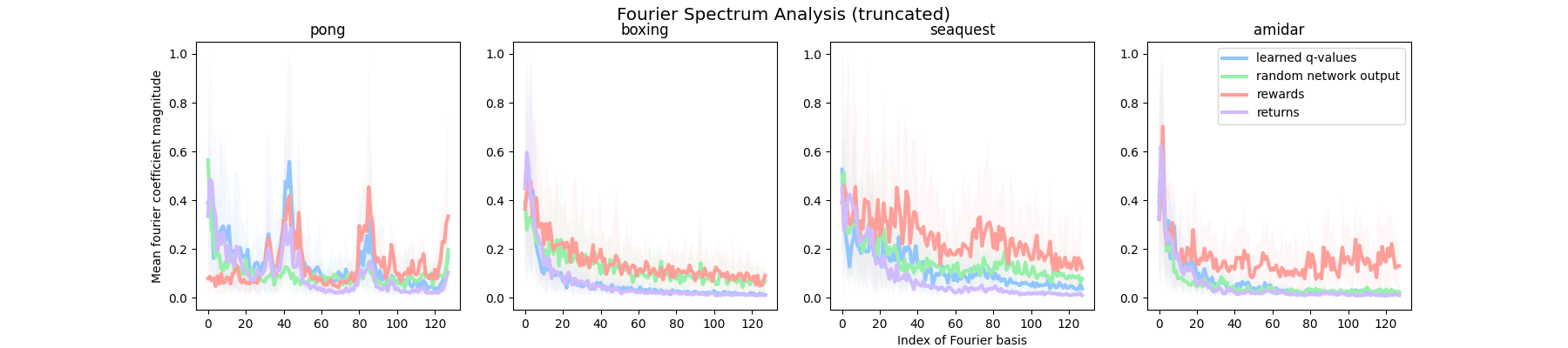}
    \caption{Fourier decomposition of Atari value functions when viewed as a function of time. We sample $k$ consecutive states from the replay buffer and compute the predicted value on each state (fixing an arbitrary action) to get a function $V: \{1, \dots, k \} \rightarrow \mathbb{R}$. We then compute the Fourier decomposition of this function. The top row shows indices $k=0 \dots 50$, while the bottom row omits the $k=0$ index (the constant function) to better illustrate the rate of decay of the spectrum of each function.}
    \label{fig:atari_fourier}
\end{figure}

\subsection{Kernel gradient descent}
\label{appx:kernel-gd}
We include an illustration of the kernel gradient descent dynamics described in Section~\ref{sec:fa_gen} in Figure~\ref{fig:kernel-dynamics}. We run our evaluations using a radial basis function (RBF) kernel of varying lengthscale, with shorter lengthscales corresponding to weaker generalization between states. While the shorter lengthscale corresponds to more stable learning dynamics and better fitting of the value function on the training set, it also induces greater value approximation error on the test states. In contrast, the longer lengthscales result in better generalization to novel test states under Monte Carlo dynamics, but result in divergence for large values of $\gamma$.
\begin{figure}
    \centering
    \includegraphics[width=0.49\linewidth]{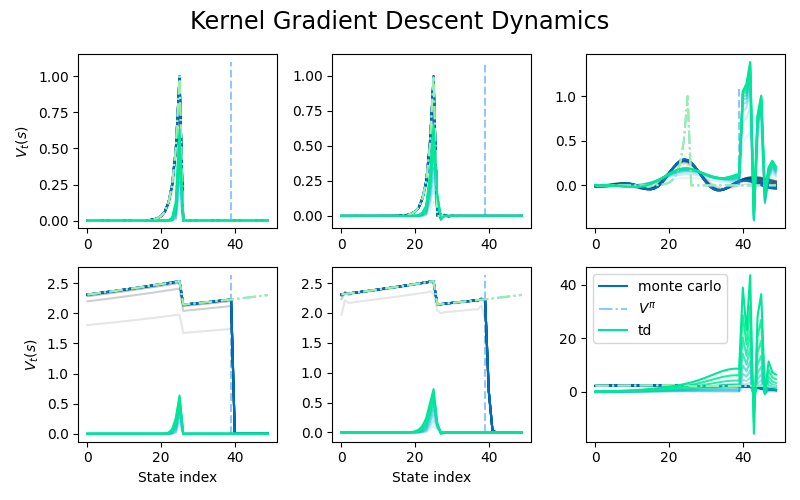} \vline 
    \includegraphics[width=0.49\linewidth]{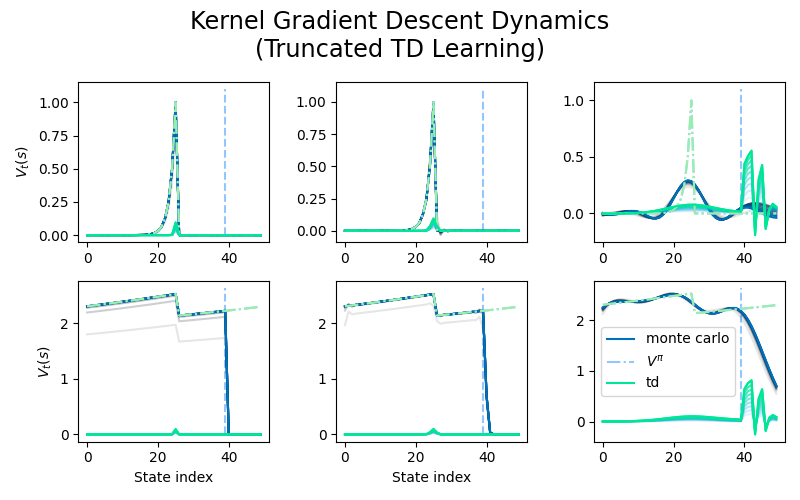}
    \caption{Numerical evaluations of kernel gradient descent with an RBF kernel. The MDP in question is a "circle MDP" whose states are integers $n \in \{1, \dots, 50\}$. We assume the agent is `trained' on states 1 to 40, and doesn't perform value function updates on the final ten states, use the policy which always takes the agent from state $n$ to $n+1 \mod 50$, and set a single reward at state 25. Each row corresponds to a different value of the discount factor $\gamma$: the top corresponds to $\gamma = 0.5$, and the bottom to $\gamma = 0.99$. Each column corresponds to the lengthscale which parameterizes the kernel, going left to right: 0.01, 1.0, and 100. The left hand side and right hand side are distinguished by the number of update steps which the TD dynamics are evaluated for. The LHS runs TD for  only 20 steps, while the RHS runs it for 100 steps. MC updates are run for 1500 steps on both figures. We see that for $\gamma = 0.99$, the larger-lengthscale kernel predictions diverge under TD dynamics, though not Monte Carlo. The Monte Carlo dynamics further nicely illustrate the trade-off between generalizing out of the training set and ability to fit the discontinuities of the value function on the training set. The larger lengthscale has lower MSE from the value function on the test set, but fails to fit the discontinuity  of the value function at the reward state. Meanwhile, the smaller lengthscales easily fit the value function on the training set but predict zero for all over states. }
    \label{fig:kernel-dynamics}
\end{figure}

Additionally, as promised in Section~\ref{sec:fa_gen}, we illustrate the role of smooth eigenfunctions in generalization in Figure~\ref{fig:kernel-generalization}. To produce this figure, we randomly generate an unweighted graph and then construct an MDP whose dynamics correspond to a random walk on this graph. We consider the generalization error of a kernel regression process where the kernel $K_S$ is of the form $ K_S(x,y) = \sum_{i \in S} v_{\lambda_i}(x) v_{\lambda_i}(y)$ for some $S \subseteq \mathrm{spec}(P^\pi)$. In the right-hand-side plot of Figure~\ref{fig:kernel-generalization}, we set $S=\{1, \dots, 20\}$, so that our analysis concentrates on smooth eigenfunctions. We then consider the generalization error of this smooth kernel when we only regress on a subset of the state space selected uniformly at random\footnote{Because the MDP-generating process is invariant to permutations of the state indices, we sample the indices $\{1, \dots, \lfloor |\states| \times \mathrm{training fraction} \rfloor \}$, and average over randomly generated MDPs. }. We study the effect of varying the size of this set, i.e. the fraction of states in the training set, in Figure~\ref{fig:kernel-generalization}, in order to quantify the degree to which additional information about the value function translates to improved generalization.
We consider three regression problems: regression on $V^\pi$, regression on the projection of $V^\pi$ onto the span of $T = \{v_1, \dots, v_{20} \}$, and $B = \{v_{n-19}, \dots, v_{n} \}$. Unsurprisingly, we see that the smooth kernel is able to improve its generalization performance as the size of the training set increases when it is set to regress $V^\pi$ or $\Pi_{T} V^\pi = V^\pi_T$. However, when the kernel regresses only on the projection of $V^\pi$ onto the non-smooth eigenvectors, we don't see a benefit of adding additional training points: because there is no information about the smooth components of the function in the targets, adding additional data points will not help to improve regression accuracy. The left hand side of the figure shows similarly that fitting local information in the form of $n$-step returns for small $n$ also does not provide the kernel with sufficient information for it to be able to extrapolate and improve its generalization error as the size of the training set increases.

\begin{figure}
    \centering
    \includegraphics[width=0.8\linewidth]{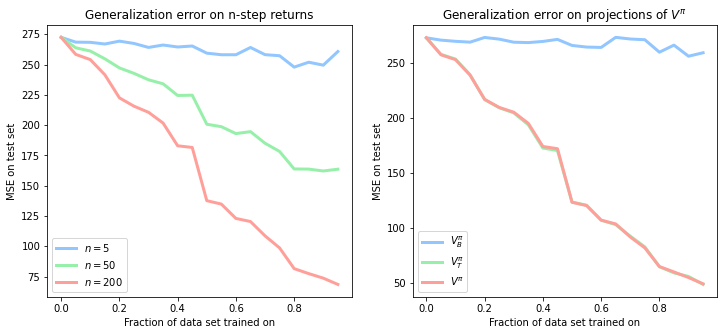}
    \caption{Generalization of predicted function under kernel regression using $n$-step return targets evaluated on a random subset of states (left), and projecting value function onto top or bottom eigenvectors of $P^\pi$ (right). We see a similar trend where for larger $n$ (corresponding to smoother targets), the kernel regression method generalizes better with increasing dataset sizes. For smaller $n$ and for the projection of $V^\pi$ onto non-smooth eigenvectors, adding additional data points doesn't improve generalization performance.}
    \label{fig:kernel-generalization}
\end{figure}
\section{Additional empirical results}
\label{apx:more-results}
\subsection{Additional value distillation results}
We consider three different types of regression to the outputs of the pre-trained network, along with two more traditional bootstrapping methods for offline RL. \texttt{Q-regression} regresses the outputs of the distilled network to those of the pre-trained network for every action. \texttt{qa-regression} only  does q-value regression on the action taken by the pre-trained agent. \texttt{adv-regression} regresses on the advantage function (computed as the q-value minus the mean over all actions) given  by the pre-trained agent; \texttt{qr} does quantile regression q-learning on the offline data; \texttt{double-q} performs a standard double q-learning update on the offline data.

We find that all of these methods obtain an initial update rank significantly below that of the pre-trained network when they begin training, which increases over time. Regression to the advantages obtains a significantly lower update rank than any other method, suggesting that the advantage function may be much smoother than the action-value function. With respect to performance on the original environment, we see that the methods which use all action values at every update obtain significantly higher performance than those which only update a single action at a time. This improvement in performance isn't mediated by an auxiliary task effect or an  increase in the network's ability to distinguish states: the advantage regression network attains low update rank  but high performance, while the qr-regression task provides a great deal of information to the representation but is not competitive with the q-regression network. 
\begin{figure}
    \centering
    \includegraphics[width=0.7\linewidth]{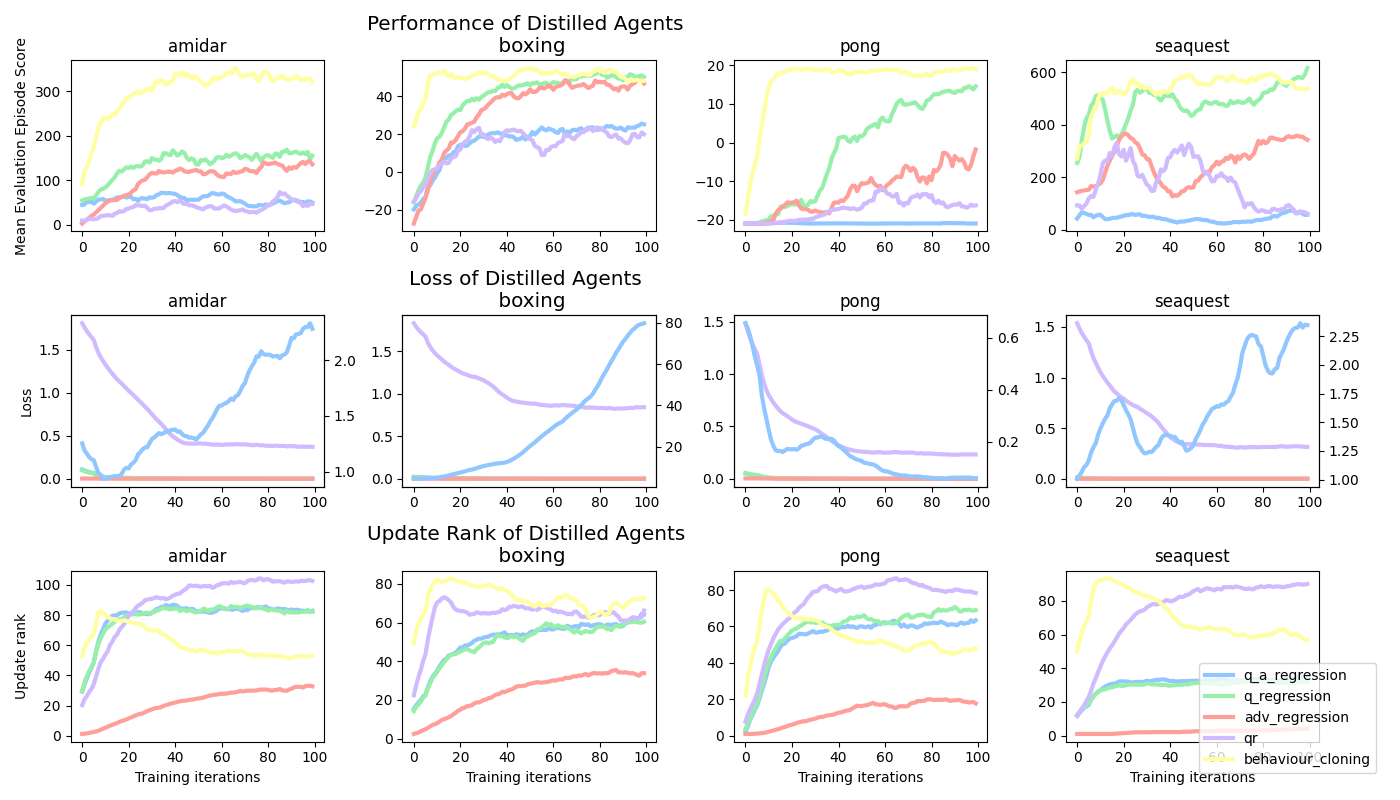}
    \caption{Results from post-training distillation on a variety of objectives. We note that advantage regression tends to exhibit the lowest update rank, with the qr agent tending to exhibit the highest update rank and the q-regression objectives falling somewhere in between. Because the behaviour cloning objective minimizes a cross-entropy loss rather than a regression loss, further investigation is required to understand how the trajectory of its update dimension differs from those of the regression objectives.}
    \label{fig:tandem-apx}
\end{figure}
\subsection{More detailed update trajectories}

We include a more detailed view of the update matrices obtained by DQN and C51 agents during the first 7 million frames of training, roughly 5\% of the training budget, in Figure~\ref{fig:updates-long}. We see that even early in training, the DQN and C51 agents both exhibit significant overfitting behaviour. Note that states are sampled uniformly at random from the replay buffer, and then assigned an index based on the output of a clustering algorithm to improve readability of the figures.

\begin{figure}
    \centering
    \includegraphics[width=0.48\linewidth]{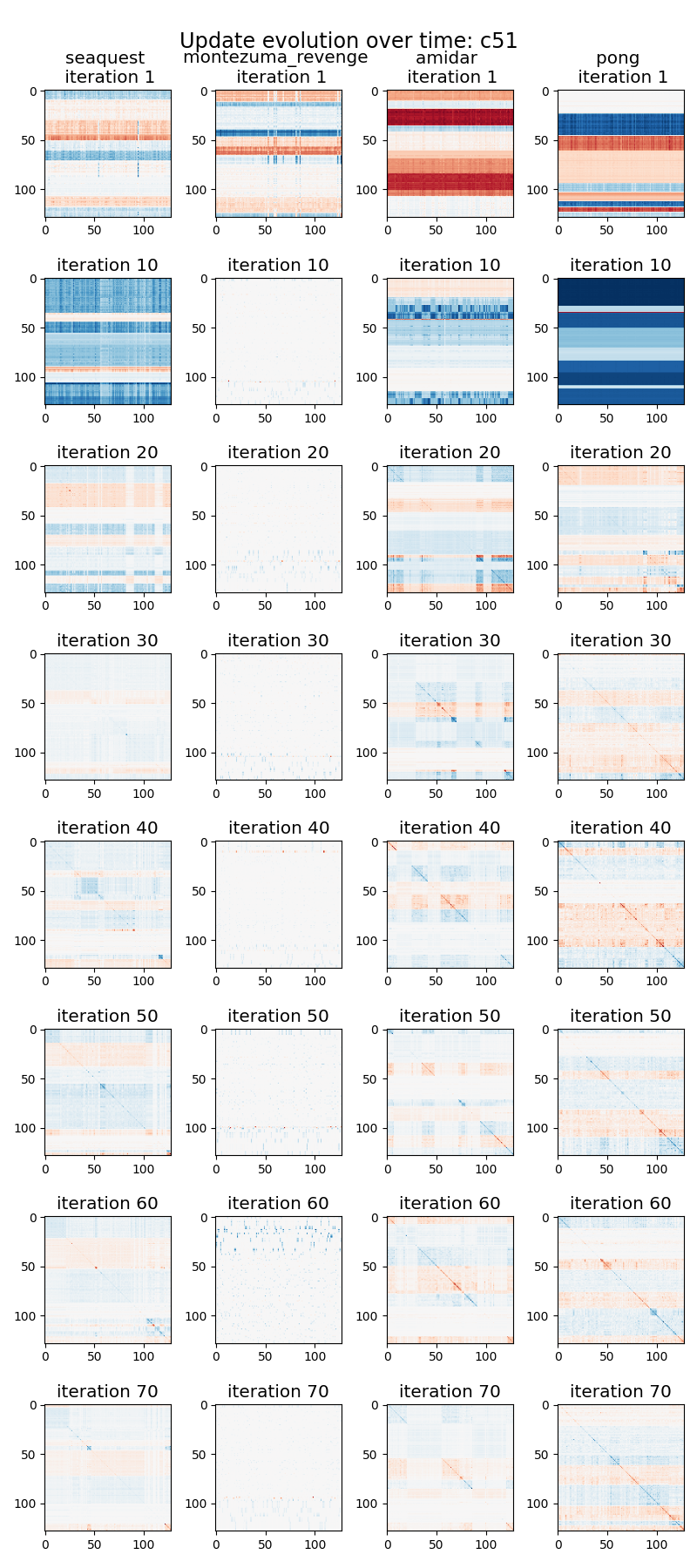}
    \vline
    \includegraphics[width=0.48\linewidth]{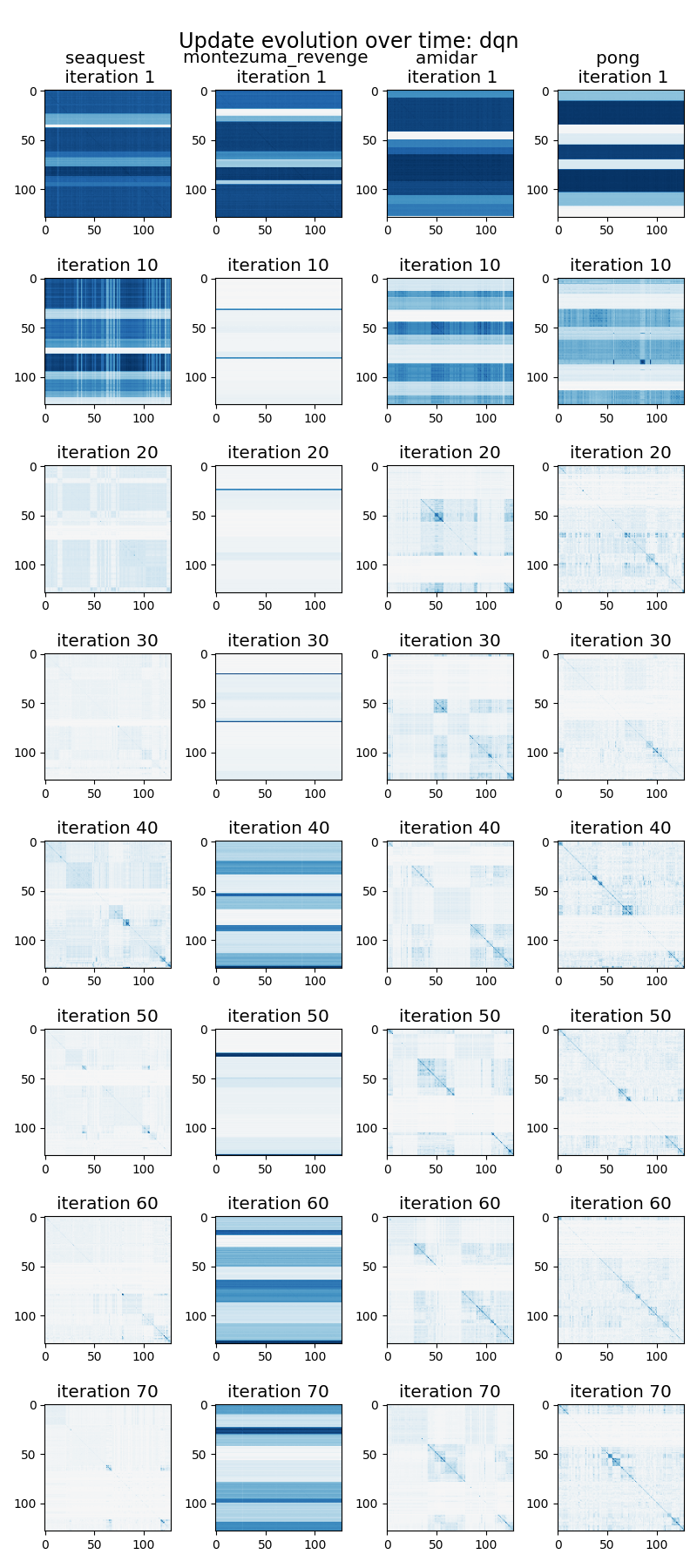}
    \caption{Update matrices for distributional and DQN agents on four games from the Atari suite, chosen to represent a range of reward densities and difficulties. Each iteration corresponds to 1e5 training frames.}
    \label{fig:updates-long}
\end{figure}